\documentclass[letterpaper]{article} % DO NOT CHANGE THIS
\usepackage{aaai25}  % DO NOT CHANGE THIS
\usepackage{times}  % DO NOT CHANGE THIS
\usepackage{helvet}  % DO NOT CHANGE THIS
\usepackage{courier}  % DO NOT CHANGE THIS
\usepackage[hyphens]{url}  % DO NOT CHANGE THIS
\usepackage{graphicx} % DO NOT CHANGE THIS
\usepackage{subfigure}

\usepackage{natbib}  % DO NOT CHANGE THIS AND DO NOT ADD ANY OPTIONS TO IT
\usepackage{caption} % DO NOT CHANGE THIS AND DO NOT ADD ANY OPTIONS TO IT
\usepackage{color}

\usepackage{theorystyle}

\usepackage{bbold}

\usepackage{newfloat}
\usepackage{listings}
\DeclareCaptionStyle{ruled}{labelfont=normalfont,labelsep=colon,strut=off} % DO NOT CHANGE THIS
\floatstyle{ruled}
\newfloat{listing}{tb}{lst}{}
\floatname{listing}{Listing}
\def\Mag{\mathrm{Mag}}
\def\reals{\mathbb{R}}

\title{Approximating Metric Magnitude of Point Sets}

\author{
     Rayna Andreeva\textsuperscript{\rm 1},
    James Ward\textsuperscript{\rm 1},
    Primoz Skraba\textsuperscript{\rm 2},
    Jie Gao\textsuperscript{\rm 3},
    Rik Sarkar\textsuperscript{\rm 1}
}
\affiliations{
    \textsuperscript{\rm 1}School of Informatics, University of Edinburgh\\
    \textsuperscript{\rm 2}Queen Mary University of London\\
    \textsuperscript{\rm 3}Rutgers University\\
    r.andreeva@sms.ed.ac.uk, rsarkar@inf.ed.ac.uk
}

\usepackage{bibentry}
\begin{document}

\maketitle

\begin{abstract}
Metric magnitude is a measure of the ``size" of point clouds with many desirable geometric properties. It has been adapted to various mathematical contexts and recent work suggests that it can enhance machine learning and optimization algorithms. But its usability is limited due to the computational cost when the dataset is large or when the computation must be carried out repeatedly (e.g. in model training). In this paper, we study the magnitude computation problem, and show efficient ways of approximating it. We show that it can be cast as a convex optimization problem, but not as a submodular optimization. The paper describes two new algorithms -- an iterative approximation algorithm that converges fast and is accurate, and a subset selection method that makes the computation even faster. It has been previously proposed that magnitude of model sequences generated during stochastic gradient descent is correlated to generalization gap. Extension of this result using our more scalable algorithms shows that longer sequences in fact bear higher correlations. We also describe new applications of magnitude in machine learning -- as an effective regularizer for neural network training, and as a novel clustering criterion.

\end{abstract}

\section{Introduction}

Magnitude is a relatively new isometric invariant of metric spaces. It was introduced to characterize ecology and biodiversity data, and was initially defined as an Euler Characteristic of certain finite categories~\cite{leinster2008euler}. Similar to quantities such as the cardinality of a point set, the dimension of vector spaces and Euler characteristic of topological spaces, Magnitude can be seen as measuring the ``effective size'' of mathematical objects. 
%For example, a set of far-away, scattered $n$ points has a magnitude of roughly $n$ but a set of $n$ points with well separated $k$ tight clusters has a magnitude roughly $k$. 
See Figure~\ref{fig:three_points_illustrative_example} for an intuition of Magnitude of Euclidean points. 
%\jie{ trying to add a bit of intuition for readers who have not seen magnitude at all -- should we add a short definition of magnitude?} \todo{Rik: may be we can refer to fig 1 for intuition. Putting definition here might be difficult without getting too formal. Also, the euclidean definition does not fit with this broader discussion...}\jie{yes the same thought.} 
It has been defined, adapted and studied in many different contexts such as topology, finite metric spaces,  compact metric spaces, graphs, and machine learning~\cite{leinster2013magnitude,leinster2013asymptotic,barcelo2018magnitudes,leinster2019magnitude,leinster2021magnitude,kaneta2021magnitude,giusti2024eulerian}.

\begin{figure}
    \centering
    \includegraphics[width=0.9\linewidth]{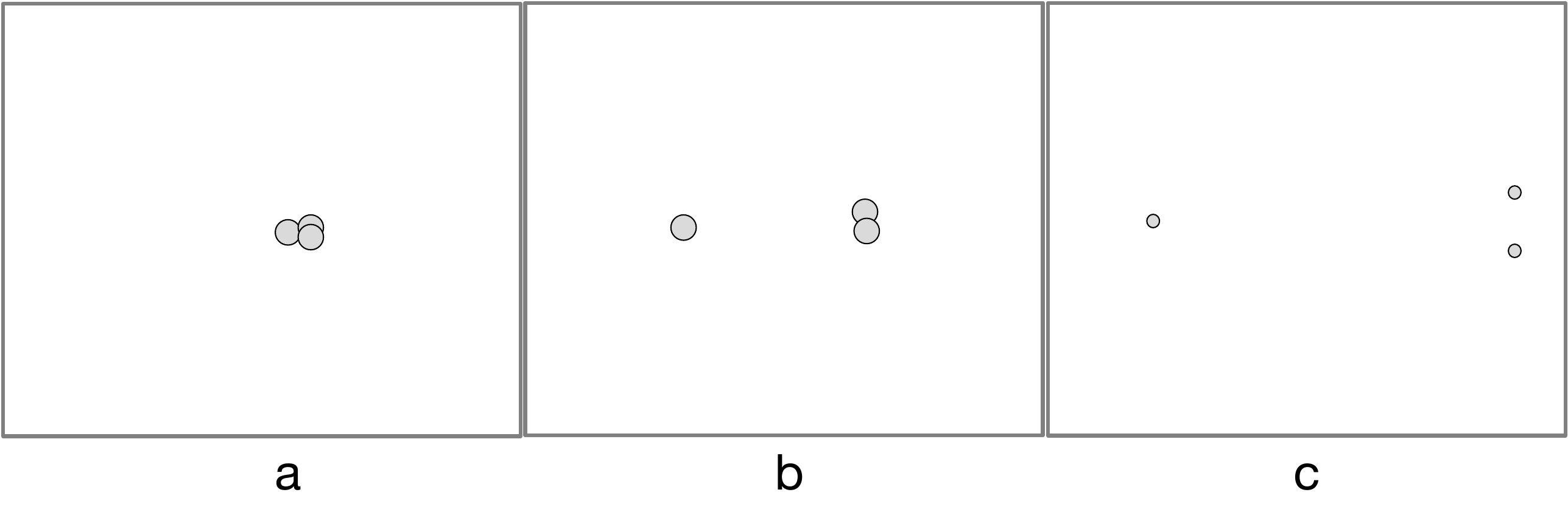}
    \caption{Consider the magnitude function of a  3-point space, visualized above at different scales. (a) For a small value of the scale parameter (e.g. $t=0.0001$), all the three points are very close to each other and appears as a single unit. This space has magnitude close to $1$. (b) At $t=0.01$ the distance between the two points on the right is still small and they are clustered together, and the third point is farther away. This space has Magnitude close to $2$ (c) When $t$ is large, all the three points are distinct and far apart, and Magnitude is 3.}
    \label{fig:three_points_illustrative_example}
\end{figure}

In machine learning and data science, the magnitude of a point cloud can provide useful information about the structure of data. It has recently been applied to study the boundary of a metric space~\cite{bunch2021weighting},  edge detection for images~\cite{adamer2024magnitude},  diversity~\cite{limbeck2024metric} and dimension~\cite{andreeva_metric_2023} of sets of points in Euclidean space, with applications in data analysis as well as generalization of models~\cite{andreeva2024topological}. Wider applications of magnitude are limited by the computation cost. For a set of $n$ points, the standard method of computing Magnitude requires inverting an $n \times n$ matrix. The best known lower bound for matrix multiplication and inversion is $\Omega(n^2\log n)$ \cite{raz2002complexity}; the commonly used Strassen's algorithm~\cite{strassen1969gaussian} has complexity $O(n^{2.81})$\footnote{Faster algorithms for matrix inversion exist, such as the Coppersmith-Winograd algorithm~\cite{COPPERSMITH1990251} with running time $O(n^{2.376})$ and Optimized CW-like algorithms with the best running time $O(n^{2.371552})$~\cite{williams24new}.}. By definition, Magnitude computation requires consideration of all pairs of input points, making it expensive for large datasets.

\myparagraph{Our contributions.} In this paper, we take the approach that for many scenarios in data science, an approximate yet fast estimate of magnitude is more likely to be useful, particularly in real-world modern applications where datasets and models are large and noisy and often require repeated computation. 

Given a point set $X\subset \reals^D$, we first show (Section~\ref{sec:convex}) that the problem of computing the magnitude $\Mag(X)$ can be seen as finding the minimum of a convex function, and thus can be approximated using suitable gradient descent methods. Next, in Section~\ref{sec:normalization} we define a new algorithm that iteratively updates a set of weights called the Magnitude weighting to converge to the true answer. This method converges quickly and is faster than matrix inversion or gradient descent. 

While avoiding inversion, both these methods need $n\times n$ matrices to store and use all pairs of similarities between points. To improve upon this setup, we take an approach of selecting a smaller subset $S\subset X$ of representative points so that $\Mag(S)$ approximates $\Mag(X)$. We first prove that Magnitude is not a submodular function, that is, if we successively add points to $S$, $\Mag(S)$ does not satisfy the relevant diminishing returns property. In fact, for arbitrarily high dimension $D$, the increase in $\Mag(S)$ can be arbitrarily large with the addition of a single point. Though in the special case of $D=1$ $\Mag(S)$ is in fact submodular, and the standard greedy algorithm~\cite{nemhauser1978analysis} for submodular maximization can guarantee an approximation of $(1-1/e)$ (Section ~\ref{sec:submodularity}). In practice, the greedy algorithm is found to produce accurate approximations on all empirical datasets -- both real-world ones and synthetic ones. This algorithm adds points to $S$ one by one; in each step it iterates over all remaining points to find the one that maximizes $\Mag(S)$. These magnitude computations are faster due to the smaller size of $S$, but the costs add up as they are repeated over $X$. 

Section~\ref{sec:hierarchy} describes an approach to speed up the approximations further. It uses properties of Magnitude such as monotonicity and growth with scale, to develop a selection method -- called Discrete centers -- that does not require repeated computation of Magnitude. It is particularly useful for computing the {\em Magnitude Function} -- which is magnitude as a function of scale, and useful in dimension computation~\cite{andreeva_metric_2023}.  This method can also easily adapt to dynamic datasets where points are added or removed. Faster estimates of magnitude allows new applications of Magnitude in machine learning. Section~\ref{sec;applications} describes use of Magnitude as a regularizer for neural network, and a clustering algorithm similar to density based clustering methods, using Magnitude as a clustering criterion.

Experiments in Section~\ref{sec:experiments} show that the approximation methods are fast and accurate. Iterative Normalization outperforms inversion for larger dataset sizes and converges fast; for the subset selection algorithms, Discrete centers approximates the Greedy Maximization approach empirically at a fraction of the computational cost. The more scalable computation allows us to produce new results in the topic of generalization, where we extend prior work on computing topological complexities based on magnitude \cite{andreeva2024topological} to a larger number of training trajectories, extending from $5 \times 10^3$ due to computational limitations to $10^4$, and observe that the correlation coefficients average Granulated Kendall ($\boldsymbol{\Psi}$) and Kendall tau ($\tau$) improve significantly with the increased number of trajectories. The new regularization and clustering methods based on Magnitude are also shown to be effective in practice. 

In the next section, we cover the technical background. Related works and discussion can be found in Section~\ref{sec:related}.

% !TEX root = main.tex

\section{Technical Background} 

This section introduces the definitions needed for the rest of the paper.

\subsection{Metric Magnitude}

For a finite metric space $(X,d)$ with distance function $d$, we define the similarity matrix $\zeta_{ij} = e^{-d_{ij}}$ for $i,j\in X$. The metric magnitude $\Mag(X,d)$ is defined~\cite{leinster2013magnitude} in terms of a {\em weighting} as follows. 

\begin{definition}[Weighting $w$] A weighting of $(X,d)$ is a function $w:X\to \reals$ such that $\forall i \in X, \sum_{j\in X}\zeta_{ij}w(j) = 1$. \label{def:weighting}
\end{definition}

We refer to the $w(i)$ as the magnitude weight of $i$, and interchangeably write it as $w_i$. 

\begin{definition}[Metric Magnitude $\Mag(X,d)$] The magnitude of $(X,d)$ is defined as $\Mag(X,d) = \sum_{i\in X}w(i)$, where $w$ is a weighting of $(X, d)$. \label{def:magnitude}
\end{definition}

The existence of a suitable $w$, and therefore magnitude of $(X, d)$ is not guaranteed in general, but it exists for finite point clouds $X\subset \reals^D$ with the ambient metric. In practice, metric magnitude is often computed by inverting the similarity matrix $\zeta$ and summing all the elements: 
\begin{align}\label{eq:magnitude}
    \Mag(X,d) = \sum_{ij}(\zeta^{-1})_{ij}.    
\end{align}
Observe that when $X$ is a finite subset of $\reals^D$, then $\zeta$ is a symmetric positive definite matrix, and the inverse exists~\cite{leinster2013magnitude}.

Magnitude is best illustrated when considering a few sample spaces with a small number of points.
For example, with a single point $a$, $\zeta_X$ is a $1 \times 1$ matrix with $\zeta_X^{-1} = 1$ and $\Mag(X) = 1$. (When distance measure $d$ is understood, such as in $R^D$, we often omit it in the notation.)

\begin{example}\label{ex:2points}
Consider the space of two points. Let $X=\{a,b\}$ be a finite metric space where $d(a,b)=d$. Then
\[
   \zeta_X=\begin{bmatrix} 1 & e^{-d}\\ e^{-d} & 1\end{bmatrix},  
\]

\noindent so that $
    \zeta_X^{-1} = \frac{1}{1 - e^{-2d}} \begin{bmatrix} 1 & -e^{-d}\\ -e^{-d} & 1\end{bmatrix},$ 

\noindent and therefore 
\[
    \mathrm{Mag}(X)=\dfrac{2 - 2e^{-d}}{1 - e^{-2d}}=\dfrac{2}{1+e^{-d}}. \label{eq:1}
\]
\end{example}

More information about a metric space can be obtained by looking at its rescaled counterparts. The resulting representation is richer, and is called the magnitude function, which we describe next.

\subsection{Scaling and the Magnitude Function}

For each value of a parameter $t\in \reals^+$, we consider the space where the distances between points are scaled by $t$, often written as $tX$. 
\begin{definition}[scaling and $tX$]
    Let $(X,d)$ be a finite metric space. We define $(tX, d_t)$ to be the metric space with the same points as $X$ and the metric $d_t(x,y) = td(x,y)$.
\end{definition}

\begin{definition}[Magnitude function]\label{def:magnitude_function}
The magnitude function of a finite metric space $(X,d)$ is the function $t \mapsto \mathrm{Mag}(tX)$, which is defined for all $t \in (0, \infty)$.
\end{definition}

This concept is best illustrated by Figure \ref{fig:three_points_illustrative_example}.

Magnitude Function is important in computing magnitude dimension, which is a determined by growth rate of $\Mag(tX)$ with respect to $t$. It is a quantity similar to fractal dimension and useful in predicting generalization of models computed via gradient descent~\cite{andreeva_metric_2023}. 

\subsection{Submodular Functions and Maximization Algorithm}

The notion of submodularity is inspired by diminishing returns observed in many real world problems. 

\begin{definition}[Submodular Function]\label{def:submodularity}
Given a set $V$, a function $f:2^V \to \reals$ is submodular set function if:
\[\forall S, T \subseteq V, f(S) + f(T) \geq f(S\cup T) + f(S \cap T).\]
\end{definition}

The definition implies that marginal utility of items or subsets are smaller when they are added to larger sets. An example is with sensor or security camera coverage, where the marginal utility of a new camera is smaller then its own coverage area as its coverage overlaps with existing cameras. 

The submodular maximization problem consists of finding a subset $S\subset V$ of a fixed size $k$ that maximizes the function $f$. It shows up in various areas of machine learning, such as active learning, sensing, summarization, feature selection and many others. See ~\cite{krause2014submodular} for a survey.
The maximization problem is NP-hard, but is often  approximated to within a factor of $1-1/e$ using a greedy algorithm~\cite{nemhauser1978analysis}.

\section{Approximation Algorithms}

We first examine algorithms that start with an arbitrary vector of weights for all points, and then iteratively adjusts them to approximate a Magnitude weighting. Then we describe methods that increase efficiency by selecting a small subset of points that have magnitude close to that of $X$. 

\subsection{Convex optimization formulation and gradient descent}

\label{sec:convex}

The problem of finding weights $w$ can be formulated as a convex optimization using the squared loss: 
\begin{align} \label{eq:convex} 
    \min_{w} \sum_i \left( \sum_j \zeta_{ij} w_j - 1\right)^2
\end{align}
This loss function is based on weighting (Definition~\ref{def:weighting}), and reflects the error with respect to an ideal weighting where for each $i$, $\sum_j \zeta_{ij} w_j$ will add up to $1$. 

This is a strongly convex optimization problem and can be addressed using methods suitable for such optimization, including gradient descent or stochastic gradient descent. 

\subsection{Iterative Normalization Algorithm}
\label{sec:normalization}

In this section we present a different algorithm that we call the Iterative Normalization Algorithm. It starts with a weight vector of all ones. Then for every point $i$, it computes the sum $G(i)= \sum_j \zeta_{ij} w_j$. For a proper magnitude weighting every $G(i)$ should be $1$, thus the algorithm simulates dividing by $G(i)$ to normalize the row to $1$, and saves $w_i \gets w_i/G(i)$. It does this in parallel for all rows (points).

\begin{algorithm}[h]
\caption{Iterative normalization algorithm for the approximation of magnitude}
\label{alg:normalization}
\textbf{Input}: The set of points $X$
\begin{algorithmic}
\State Initialise $w_i = 1$ for all $i\in X$
\While{not converged}
    \State Compute $G(i)= \sum_j \zeta_{ij} w_j$ for all $i\in X$
    \State Update $w_i = w_i/G(i)$ for all $i\in X$ 
\EndWhile    
\end{algorithmic}
\end{algorithm}

Observe that compared to matrix inversion, which has a complexity of $O(n^{2.371552})$,
%$O(n^{2.81})$, 
the iterative normalization uses $O(n^2)$ per iteration, and achieves useable accuracy in relatively few iterations. In this problem, unlike usual optimization problems, we in fact know the optimum value of the loss for each point, and as a result we can use this approach of pushing the parameters toward this minimum value.

A caveat is that this algorithm produces a weighting that consists of positive weights. While individual magnitude weights can in principle be negative, Magnitude of a point cloud is always positive and in our experiments, the algorithm always finds a weighting whose sum converges toward the true magnitude. In this context, note that positive weights have been found to be relevant in predicting generalization of neural networks. See~\cite{andreeva2024topological}. 

\subsection{Approximation via greedy subset selection}\label{sec:submodularity}

To approximate more efficiently, we can attempt to identify a subset $S$ of points that approximate the magnitude of $X$. Magnitude increases monotonically with addition of points to $S$~\cite{leinster2013magnitude}, which suggests approximation via algorithms that greedily add points to $X$ similar to Nemhauser's submodular maximization~\cite{nemhauser1978analysis}. However, magnitude of a point set is not quite submodular, and thus the approximation guarantees do not carry over. 

The non-submodularity can be seen in the following counterexample. Let $e_1,...,e_D$ be the standard basis vectors of $\mathbb{R}^D$, so $e_1 = (1,0,...,0)$ etc. Let $t > 0$ be a real number and $te_1,...,te_D$ be the scaled basis vectors of $\mathbb{R}^D$, so $te_1 = (t,0,0...,0)$ etc. Consider $X = \{te_1,-te_1,...,te_D,-te_D\}$, with the usual metric. Thus $X$ consists of the points on the axes of $\mathbb{R}^D$ that are a distance of $t$ away from the origin. For a numerical example: when $t=5$ and $D=500$, we get $\Mag(X\cup \{0\}) - \Mag(X) \approx 7.18$. Thus, while the magnitude of a single point (origin) is $1$ by itself, adding it to $X$ produces an increase far greater than $1$.

This construction can be generalised to higher dimensions $D$, and behaves as follows in the limit:

\begin{theorem} \label{thm:non-submodularity}
    Let $X = \{te_1,-te_1,...,te_D,-te_D\}$ be a set of points in $\mathbb{R}^D$ as described above. Then in the limit:
    \[ \lim_{D \to \infty} \left( \Mag(X\cup \{0\}) - \Mag(X) \right) = \frac{(e^t-e^{t\sqrt{2}})^2}{e^{2t}-e^{t\sqrt{2}}}. \]
\end{theorem}

\subsubsection{Greedy set selection algorithm}

While the theorem above implies that submodularity does not hold in general for magnitude, our experiments suggest that in practice, Nemhauser's algorithm~\cite{nemhauser1978analysis} adapted to Magnitude achieves approximation rapidly. A version of this idea can be seen in Algorithm~\ref{alg:greedy}.

\begin{algorithm}[tb]\label{alg:greedy}
\caption{Greedy algorithm for the computation of original magnitude}
\label{alg:greedy}
\textbf{Input}: The set of points $S$\\
\textbf{Parameter}: Tolerance $k$\\
\textbf{Output}: The approximated total magnitude and the maximising set $S'$
\begin{algorithmic} %[1] enables line numbers
\State Initialise $S'$ to be the empty set
\State Add a random element $s_1$ from $S$ to $S'$.
\While{$Mag(S') < (1-k)*Mag(S' \setminus s_i|)$ (The previous computation of magnitude is within the tolerance parameter)}
\State Find the element $s_i$ in $S \setminus S'$, maximising $Mag(S' \cup s_i)$
\State Add $s_i$ to $S'$
\EndWhile
\State \textbf{return} $S', Mag(S')$
\end{algorithmic}
\end{algorithm}

\begin{figure}
    \begin{tabular}{cc}
        \includegraphics[width=0.5\columnwidth]{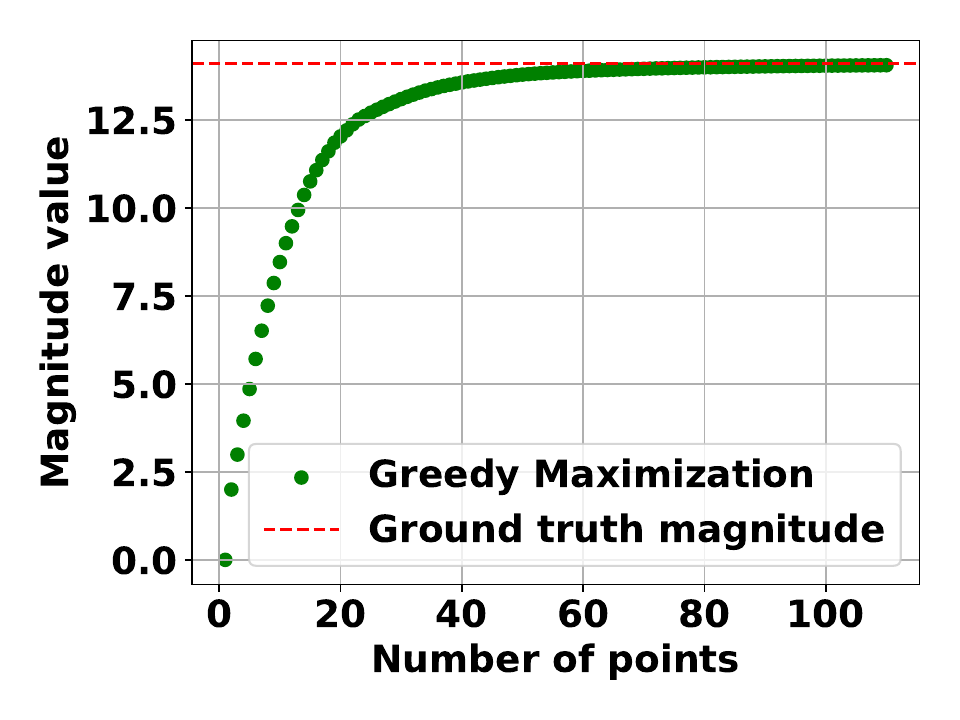} & \includegraphics[width=0.5\columnwidth]{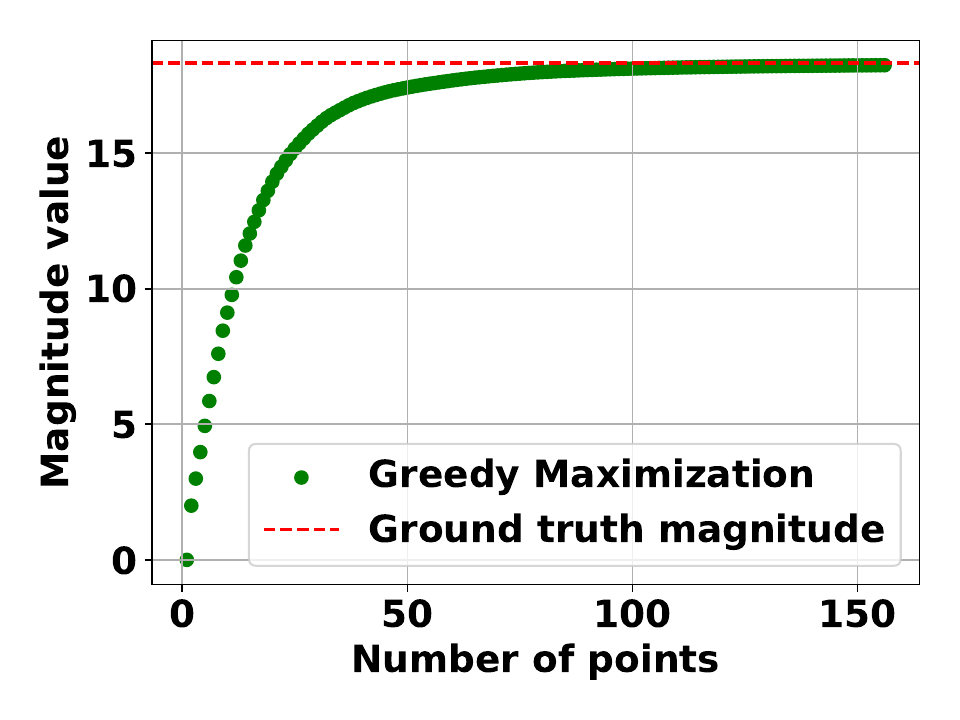} %& \includegraphics[width=0.3\textwidth]{Figures/Submodularity/Better_font/submodularity_experiment_10000.pdf} \\
        \\
        (a) & (b) \\
    \end{tabular}
    
    \caption{\textbf{Greedy algorithm approximates magnitude with small number of points.} Plot (a) shows magnitude approximation of a Gaussian blobs, 3 centers, with 500 points. Plot (b) shows Gaussian blobs with 3 clusters and $10^4$ points.}
    \label{fig:magnitude_points_growth}
\end{figure}

In certrain restricted cases, submodularity can be shown: 
\begin{theorem}
\label{thm:submod-R}
$\Mag(X)$ is submodular when $X\subset\reals$. 
\end{theorem}

Thus, when $X\subset \reals$, the greedy approximation of $(1 - 1/e)$ holds. 

\subsection{Discrete Center Hierarchy Algorithm}\label{sec:hierarchy}

The computational cost of the greedy algorithm arises from the need to repeatedly compute magnitude at each greedy step to examine $\Omega(n)$ points and compute magnitude each time to find the next point to add. To avoid this cost, we propose a faster approximation method. 

In addition to being monotone increasing with addition of points in $X$, the $\Mag(tX)$ also grows with $t$, and at the limit $\lim_{t\to \infty} \Mag(tX) = \#X$. (where $\#X$ is the number of points in $X$)~\cite{leinster2013magnitude}. Therefore, point sets with larger distances between the points will have larger magnitude. Thus an iterative subset selection algorithm that prefers well separated points is likely to increase the estimate faster toward the true magnitude

This effect is achieved using Algorithm~\ref{alg:hierarchy}, which creates a hierarchy of discrete centers and uses them to successively approximate magnitude. The hierarchy is constructed as a sequence $S_0, S_1, S_2,\dots$ of independent covering sets. Given a set $S$, a subset $s$ is a minimal independent covering set of radius $r$, if it satisfies the following properties: (1): for every $x \in S$, there exists $y \in s$ such that $d(x,y) \leq r$ (2): $\forall x, y\in s, d(x, y) > r$ and (3) $s$ is minimal with respect to these properties, that is removing any point from $s$ will violate the first property. With this in mind, we can construct the hierarchy as follows:

\begin{algorithm}[h]
\caption{Discrete Center Hierarchy construction}
\label{alg:hierarchy}
\begin{algorithmic} 
     \State \textbf{Input: } $(X, d)$. 
     \State $S_0 = X$
     \State $S_i \gets \emptyset$ for $i=1,2,...$
     \For{$i=1,2,...$}
         \State Select $S_i \subseteq S_{i-1}$ where $S_i$ is a minimal independent covering set of $S_{i-1}$ of radius $2^{i-1}$
    \EndFor
 \end{algorithmic}
 \end{algorithm}

 The hierarchy will have a height of at most $h=\log_2(\max_{x,y\in X} d(x,y))$, that is, log of the diameter of $X$. This hierarchy is used to successively approximate magnitude by traversing it from the top to the bottom. That is, starting from $s=\emptyset$, we first add points in $S_h$ to $s$, followed by those in $S_{h-1}, S_{h-2}$ etc, with $\Mag(s)$ increasing toward $\Mag(X)$. 

Observe that when computing Magnitude function (Definition~\ref{def:magnitude_function}) which requires computation for several values of $t$, this same sequence can be used for approximations at all the scales. Experiments described later show that a small number of points in this sequence (from the top few levels) suffices to get a good approximation of magnitude. 

 \begin{figure*}
    \begin{tabular}{cccccc}
      \includegraphics[width=0.15\textwidth]{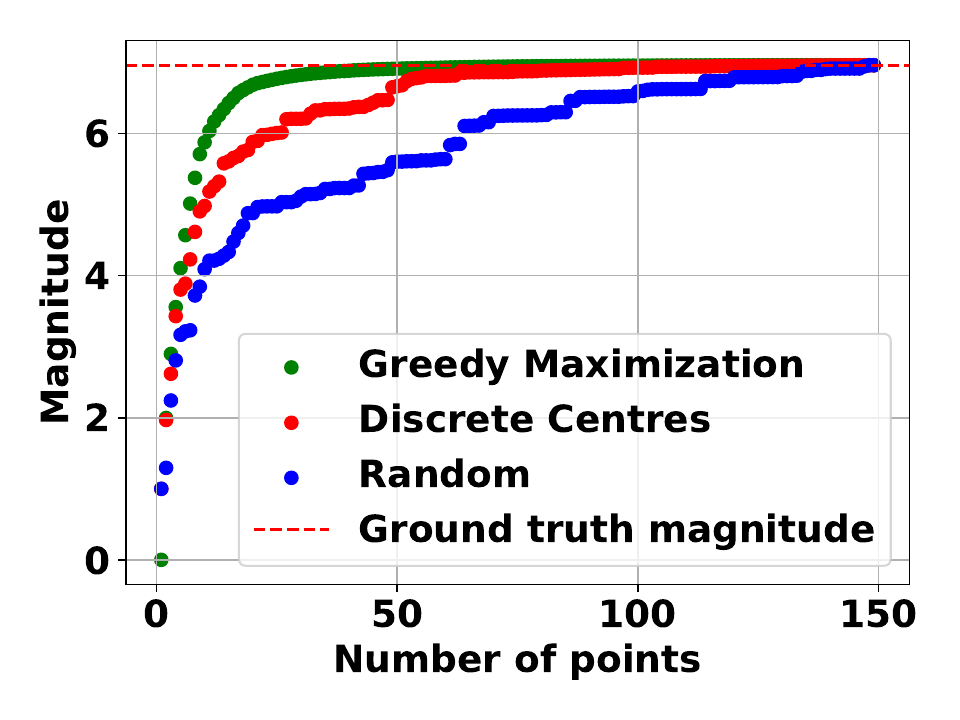}
      &
      \includegraphics[width=0.15\textwidth]{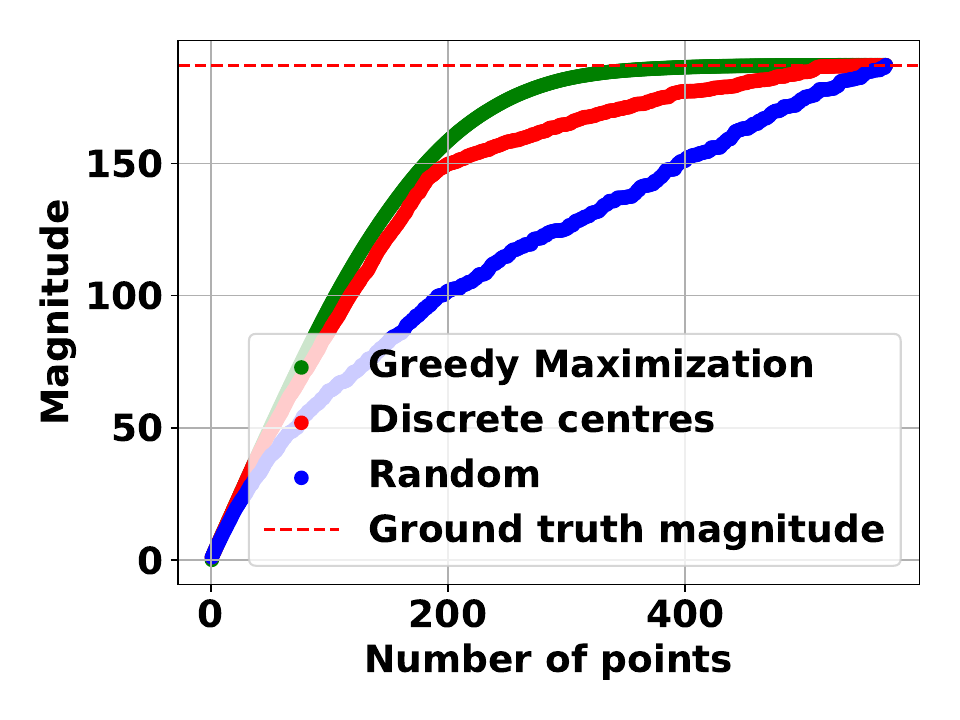}
        &
      \includegraphics[width=0.15\textwidth]{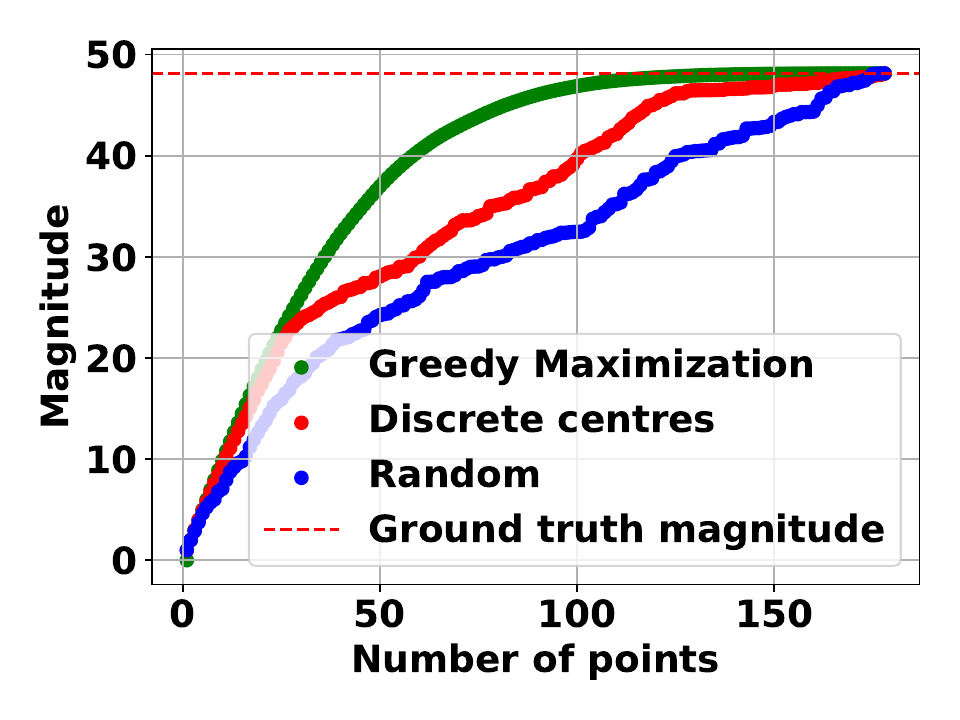}& \includegraphics[width=0.15\textwidth]{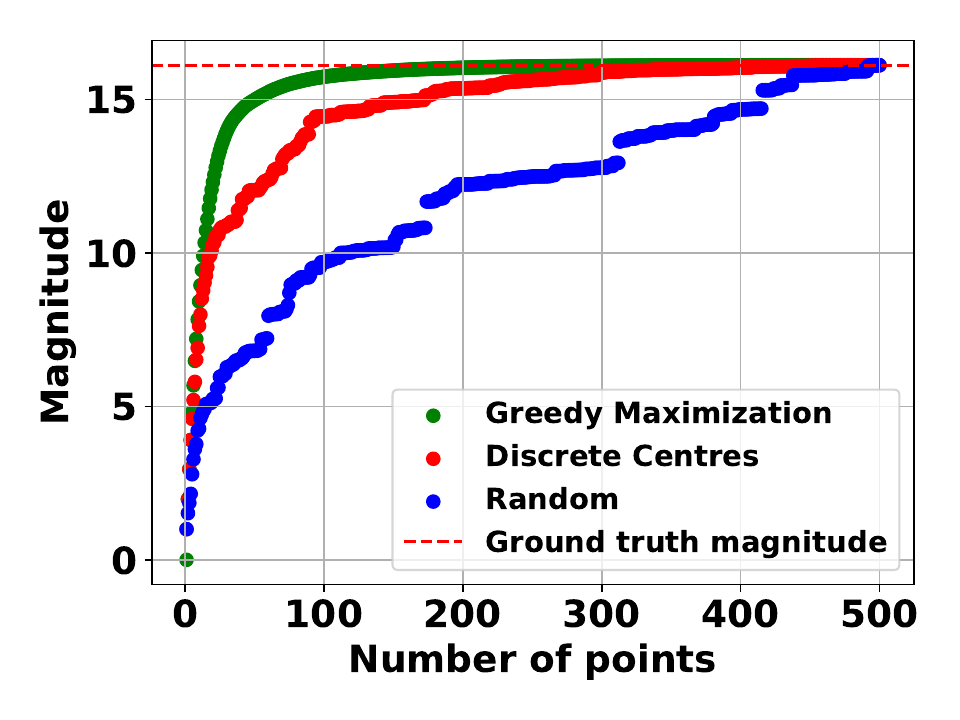}  &
        \includegraphics[width=0.15\textwidth]{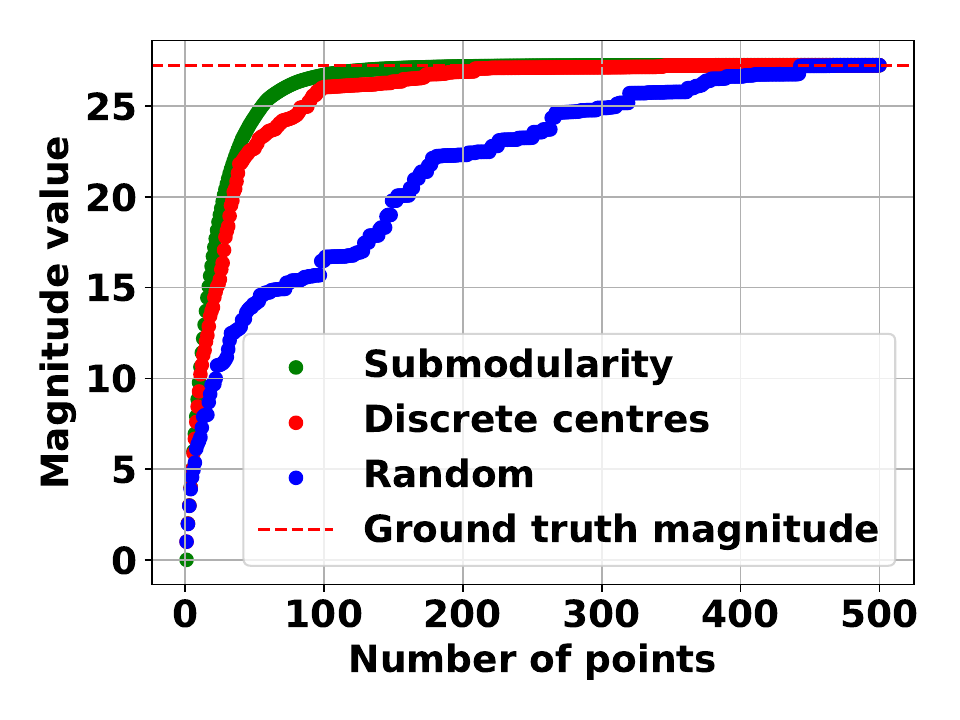} &
        \includegraphics[width=0.15\textwidth]{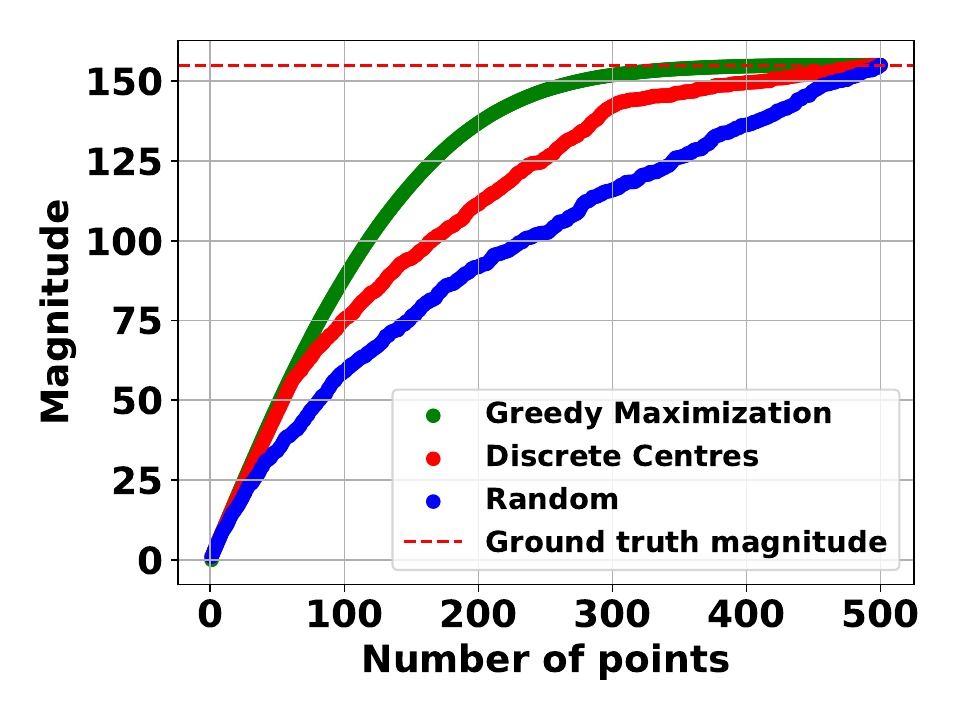}
        \\
        \\
        (a) & (b) & (c) & (d) & (e) & (f) \\
    \end{tabular}
    
    \caption{\textbf{Discrete centers are close to Greedy Maximization at a fraction of the computational cost and better than random.} In plot (a) we have the Iris dataset, in plot (b) the Breast cancer dataset, in plot (c) the Wine dataset. In the remaining plots, we see subsamples of size 500 for popular image datasets: (d)   MNIST, (e) CIFAR10 and (f) CIFAR100.}
    \label{fig:magnitude_sklearn_datasets_subsets}
\end{figure*}

\myparagraph{Incremental updates to the hierarchy.} This hierarchy can be efficiently updated to be consistent with addition or removal of points. When a new point $q$ is added, we traverse top down in the hierarchy searching for the center in $S_i$ within distance $2^{i-1}$ to $q$. When such a center does not exist, we insert $q$ to be a center in $S_j$ for all $j\leq i$. With a data structure that keeps all centers of $S_i$ within distance $c\cdot 2^{i-1}$ for some constant $c$, we could implement efficient `point location' such that insertion takes time proportional to the number of levels in the hierarchy. If a point $q$ is deleted from the hierarchy, we need to delete $q$ from bottom up. At each level $i$ if there are centers of $S_{i-1}$ within distance $2^{i-1}$ from $q$, some of them will be selectively `promoted' to $S_i$ to restore the property. For more detailed description of a similar geometric hierarchy, see~\cite{Gao2006-bq}.

\section{Applications in Machine Learning} 
\label{sec;applications}
Here we describe the use of magnitude in two novel applications: as a regularization strategy for neural networks and for clustering.

\begin{algorithm}
\caption{Magnitude Clusterer}\label{alg:mag-cluster}
\begin{algorithmic}
\State Let $X$ be a set of points (scaled so the average pairwise distance is 1) and $t \geq 0$ be some threshold.
\State Initialise $R = X \setminus \{a\}$ and $C = \{\{ a\}\}$ for some random point $a$.
    \While{$R \neq \varnothing$}
    \State Initialise \textbf{best increase} = $\infty$ and \textbf{best point} = $\varnothing$, \textbf{best cluster} = $\varnothing$.
    \For{$b \in R$, $c \in C$}
    \State Set \textbf{increase} = $\Mag(c \cup \{b\}) - \Mag(c)$
    \If{\textbf{increase} $<$ \textbf{best increase}}
    \State \textbf{best increase} = \text{increase}
    \State \textbf{best point} = b
    \State \textbf{best cluster} = c
    \EndIf
    \EndFor
    \If{\textbf{increase} $<$ t}
    \State Replace $c \in C$ with $c \ \cup \{$ \textbf{best point} $\}$.
    \Else
    \State Add \{ \textbf{best point} \} to $C$.
    \EndIf
    \State Remove \textbf{best point} from $R$.
    \EndWhile
    \State return $C$
\end{algorithmic}
\end{algorithm}

\subsection{Neural Network Regularization}
\label{sec:netrowk_regularization}

Large neural network weights can be an indicator of overfitting to noise in the training data. Methods like weight decay add a term to the model's loss function to penalise large weights. We use Magnitude of the weights as a regulariser term. If the weight parameters are given by a vector $p$, where each $p_i\in \reals$, then the magnitude of this metric space $(p, \reals)$ with the ambient metric of $\reals$ is submodular (Theorem~\ref{thm:submod-R}) with guaranteed approximation of $1-1/e$.

Specifically, we use the following algorithm to estimate magnitude. First select 1000 randomly chosen weights of the network, and then add the network weights with the smallest and largest values. As the set of the smallest and largest weights is the set of two weights with the largest possible magnitude, these points will be returned by the initial execution of the Greedy Maximization algorithm which, as magnitude is submodular on the real line, has a theoretical guarantee of performance. Then select a random subset of the remaining weights.

\subsection{Clustering}

Inspired by the greedy approximation algorithm for submodular set functions, we propose a novel magnitude-based clustering algorithm. The key idea behind this algorithm is that, given a pre-defined set of clusters, if a new point belongs to one of those clusters then its inclusion in the cluster should not cause the magnitude of the cluster to increase significantly. Thus the algorithm works as follows: In every round, the algorithm tries to find a point $b$ that is coherent with an existing cluster $c$, where coherence is measured as the change in magnitude of $c$ being below some threshold $t$ when adding $b$ to $c$. If no such point-cluster pair  can be found, then the algorithm initializes $b$ as a new cluster. The details are in Algorithm~\ref{alg:mag-cluster}. 

Good thresholds can be found by carrying out magnitude clustering over a range of threshold values and monitoring the number of clusters. The cluster counts that persist over a range of threshold values are likely to be represent natural clusterings of the data. Selecting the most persistent count is natural way to determine clustering without any other parameter.

% temporary fix to estimate space
% \vspace{-5.0cm}
\section{Experiments}
\label{sec:experiments}
Experiments ran on a NVIDIA 2080Ti GPU with 11GB RAM and Intel Xeon Silver 4114 CPU. We use \texttt{PyTorch}'s GPU implementation for matrix inverison. The SGD experiments used a learning rate of $0.01$ and momentum of $0.9$.

\subsection{Accuracy and computation cost comparison}

\subsubsection{Iterative algorithms}

In Figure \ref{fig:iterative_algo_comparison_random} we see a comparison of the iterative algorithms on points sampled from $\mathcal{N}(0,1)$ in  $\mathbb{R}^2$ with $10^4$ points. We observe that the Iterative Normalization algorithm is faster than Inversion and SGD in plot (a) and it only needs a few iterations (less than 20) to converge as seen in plot (b), while GD/SGD takes a longer number of iterations. In plot (c) we see the convergence performance over 100 different runs, and again we note that Iterative Normalization converges fast, while GD requires a larger number of iterations.
\begin{figure}
    \begin{tabular}{ccc}
        \includegraphics[width=0.14\textwidth]{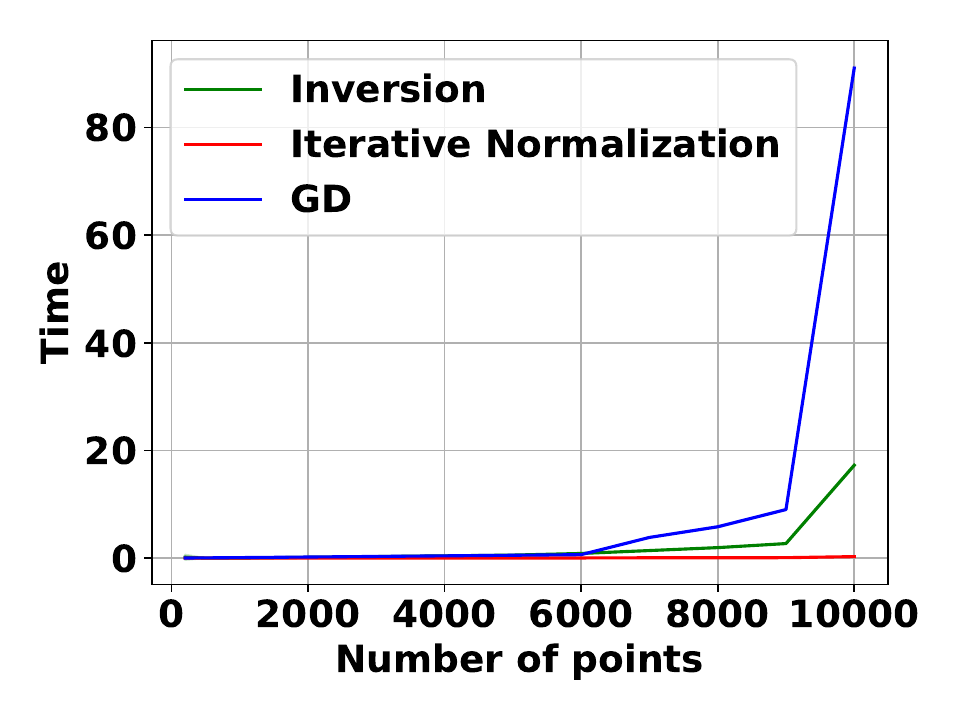} &         \includegraphics[width=0.14\textwidth]{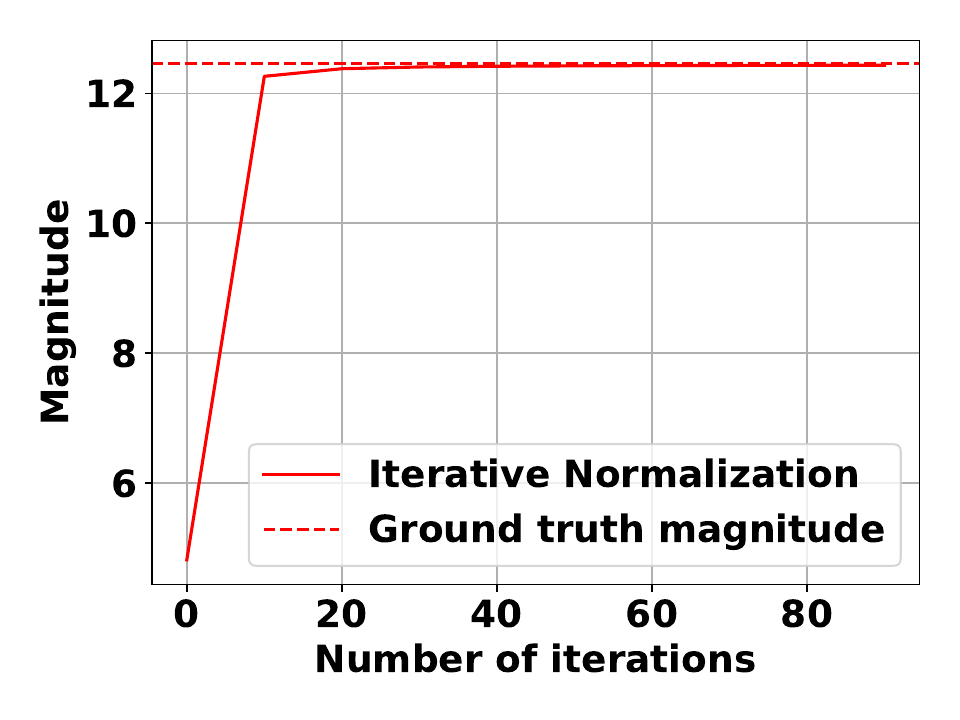} &
        \includegraphics[width=0.14\textwidth]{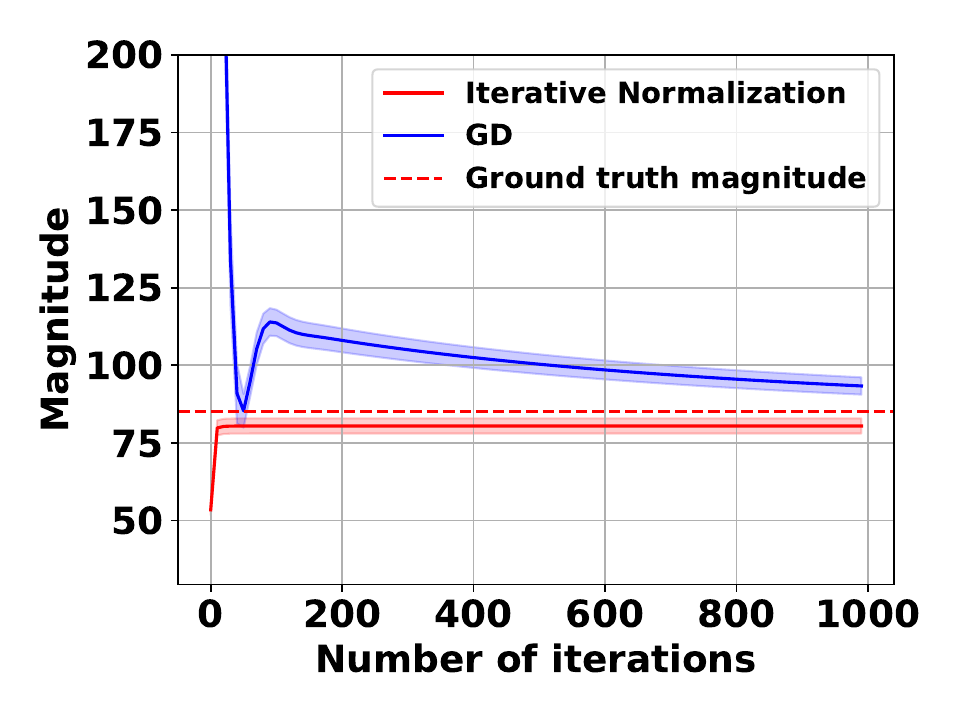} 
        \\
        (a) & (b) & (c) \\
    \end{tabular}
    \caption{\textbf{Iterative algorithms comparison} Comparison of Inversion, Iterative Normalization and GD (a) Mean and standard deviation over 10 different runs, with 50 iterations of both iterative algorithms. (b) Number of iterations for convergence of Iterative Normalization for a randomly generated sample of 10000 points. (c) Iterative Normalization vs GD. Iterative Normalization converges fast, GD takes a longer number of iterations. 100 runs. Comparison on larger point sets in supplementary materials.} 
    \label{fig:iterative_algo_comparison_random}
\end{figure}

\subsubsection{Subset selection algorithms}

Figure \ref{fig:subset_random_comparison}, shows a comparison of the subset selection algorithms on a randomly generated dataset with $10^4$ points sampled from $\mathcal{N}(0,1)$ in $\mathbb{R}^2$.

Figure \ref{fig:magnitude_sklearn_datasets_subsets} shows the performance of the subset selection algorithms for a number of \texttt{scikit-learn} datasets (Iris, Breast Cancer, Wine) and for subsamples of MNIST, CIFAR10 \cite{krizhevsky_cifar-10_2014} and CIFAR100 \cite{cifar100}. We note that the Greedy Mazimization performs the best, but Discrete Centers produces a very similar hierarchy of points. Selecting points at Random does not lead to an improvemnet in a sense that you need to approach the cardinality of the set to get a good enough approximation of magnitude.

\begin{figure}
    \begin{tabular}{ccc}
        \includegraphics[width=0.2\textwidth]{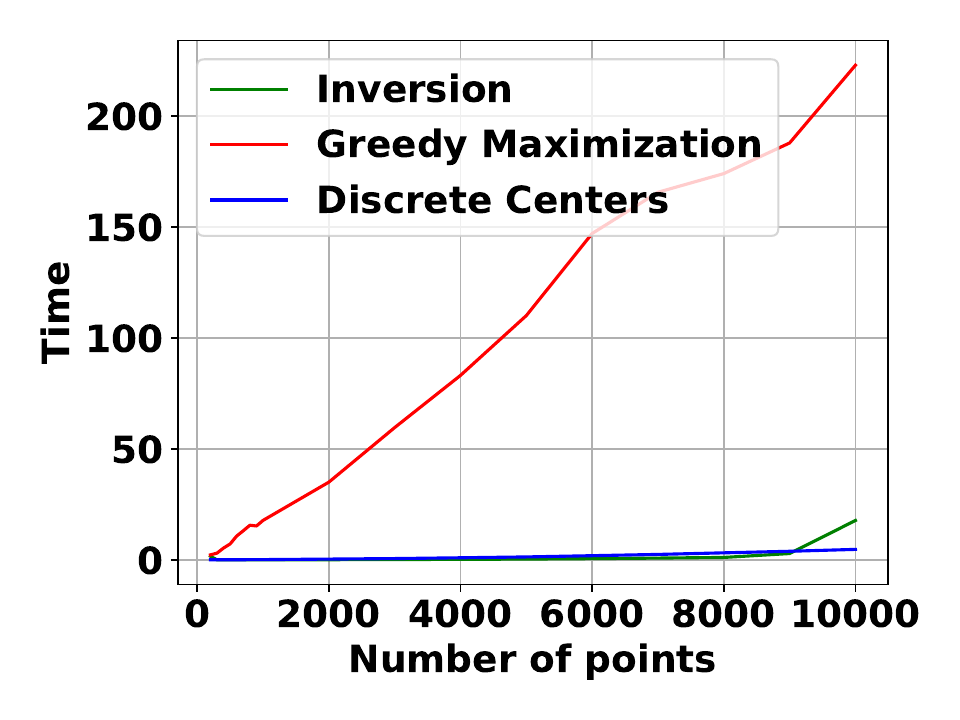} &         \includegraphics[width=0.2\textwidth]{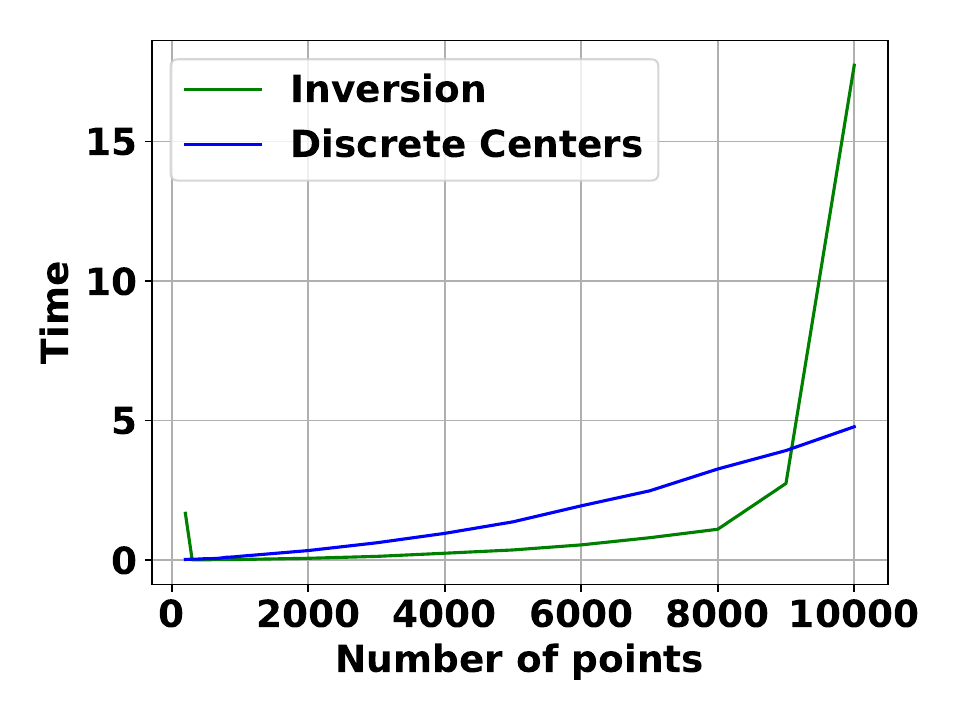} 
        \\
        (a) & (b) \\
    \end{tabular}
    \caption{\textbf{Subset selection algorithms comparison} (a) Time taken for Inversion, Greedy Maximization and Discrete Centers to execute. (b) zoom on the performance of Inversion and Discrete Centers, and note that Discrete Centers performs better as the number of points increases. Comparison on larger datasets in supplementary materials.} 
    \label{fig:subset_random_comparison}
\end{figure}

\subsection{Applications in ML}

\subsubsection{Training trajectories and generalization}
It has been shown that Magnitude and a quantity derived from Magnitude called Positive magnitude ($\mathrm{PMag})$, consisting of positive weights) are important in bounds of worst case generalization error. The method relies on computing a trajectory by taking $n$ steps of mini-batch gradient descent after convergence, and computing the Magnitude of corresponding point set on the loss landscape. See~\cite{andreeva2024topological} for details. 

The experiments up to now have been limited to training trajectories of at most size $5 \times 10^3$ due to computational limitations. Our faster approximation methods can allow us to verify the results on larger trajectories. 

We denote by  $\mathrm{Mag_n}$ and $\mathrm{PMag_n}$ the relevant quantities trjectories of length $n$. 
Size upto $5000$ have been considered in the original paper. We extend to sizes of $7000$ and $10000$. We use ViT \cite{touvron2021training} on CIFAR10, ADAM optimizer \cite{kingma2017adam}, and perform the experiment over a grid of 6 different learning rates and 6 batch sizes, where the learning rate is in the range $[10^{-5}, 10^{-3}]$, and the batch size is between $[8,256]$ resulting in 36 different experimental settings. 

The results can be found in Table \ref{table:mag-trajectories-more}, showing a number of correlation coefficients relevant for generalization \cite{gastpar2023fantastic} between generalization gap and $\mathrm{Mag_n}$ and $\mathrm{PMag_n}$ for $n=\{5000, 7000, 10000\}$. We use the granulated Kendall's coefficients ($\boldsymbol{\psi}_{\text{lr}}$ and $\boldsymbol{\psi}_{\text{bs}}$ are the granulated Kendall coefficient for the learning rate and for batch size respectively, and $\boldsymbol{\Psi}$ is the Average Kendall coefficient, which is the average of $\boldsymbol{\psi}_{\text{lr}}$ and $\boldsymbol{\psi}_{\text{bs}}$ \cite{gastpar2023fantastic}), which are more relevant than the classical Kendall's coefficient for capturing causal relationships.

We observe that all correlation coefficients improve with the increase of trajectory size. In particular, the Kendall tau coefficient and the Average Granulated Kendall coefficient increases by $0.14$ for $\mathrm{Mag_{10000}}$ compared to $\mathrm{Mag_{5000}}$, and by $0.09$ for $\mathrm{PMag_{10000}}$. Similarly, Kendall tau improves by $0.10$ for $\mathrm{Mag_{10000}}$ and $0.05$ for $\mathrm{PMag_{10000}}$. This is an interesting result which needs to be investigated further for more models and datasets. Further visualisation results can be seen in Figure \ref{fig:magnitude_traj_ours}, where we see how the proposed quantities change with the generalization gap, and when more trajectories are considered.

\begin{table}
\centering
    \begin{tabular}{|l|l|l|l|l|}
    \hline
        Metric & {\footnotesize$\boldsymbol{\psi}_{\text{lr}}$} & {\footnotesize$\boldsymbol{\psi}_{\text{bs}}$} & {$\boldsymbol{\Psi}$ } &  {$\tau$} \\
         \hline
         $\mathrm{Mag_{5000}}$ %\cite{andreeva2024topological} 
         & 0.68 & 0.62 & 0.65 & 0.64 \\
         $\mathrm{Mag_{7000}}$ & 0.71 & 0.77 & 0.74 & 0.69 \\
         $\mathrm{Mag_{10000}}$ & \textbf{0.75} & \textbf{0.82} & \textbf{0.79} & \textbf{0.74} \\
         \hline
         $\mathrm{PMag_{5000}}$ %\cite{andreeva2024topological} 
         & 0.91 & 0.67 & 0.79 & 0.85 \\
         $\mathrm{PMag_{7000}}$ & 0.93 & 0.73 & 0.83 & 0.88 \\
         $\mathrm{PMag_{10000}}$ & \textbf{0.97} & \textbf{0.79} & \textbf{0.88} & \textbf{0.90} \\
         \hline
    \end{tabular}
    \caption{Generalization gap correlation improvement using an increasing number of points.  $\tau$ is Kendall tau.}
    \label{table:mag-trajectories-more}
\end{table}

\begin{figure*}
    \begin{tabular}{cccc}
        \includegraphics[width=0.25\textwidth]{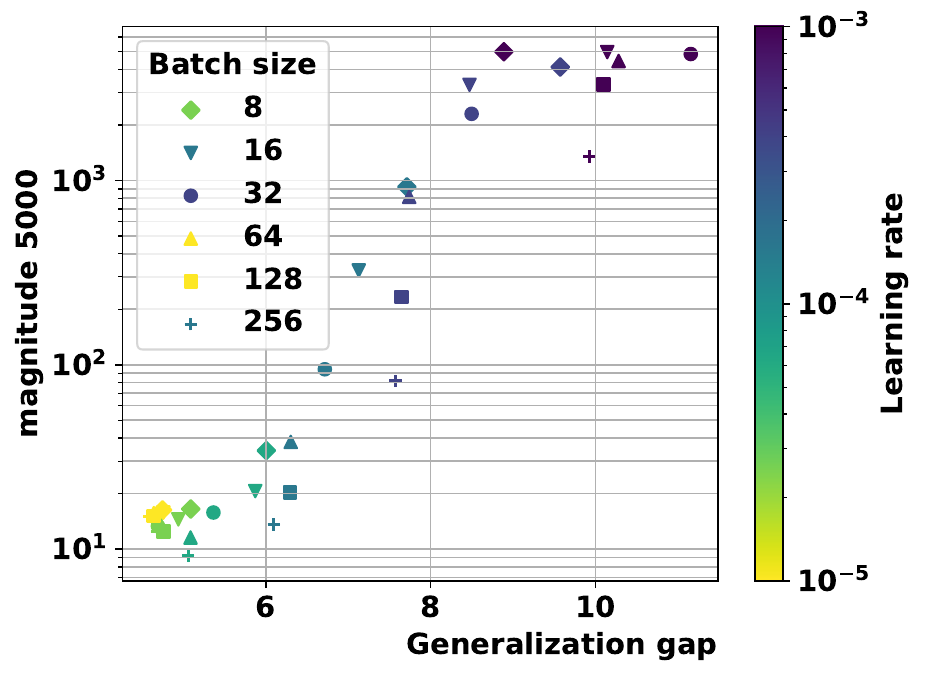} & \includegraphics[width=0.25\textwidth]{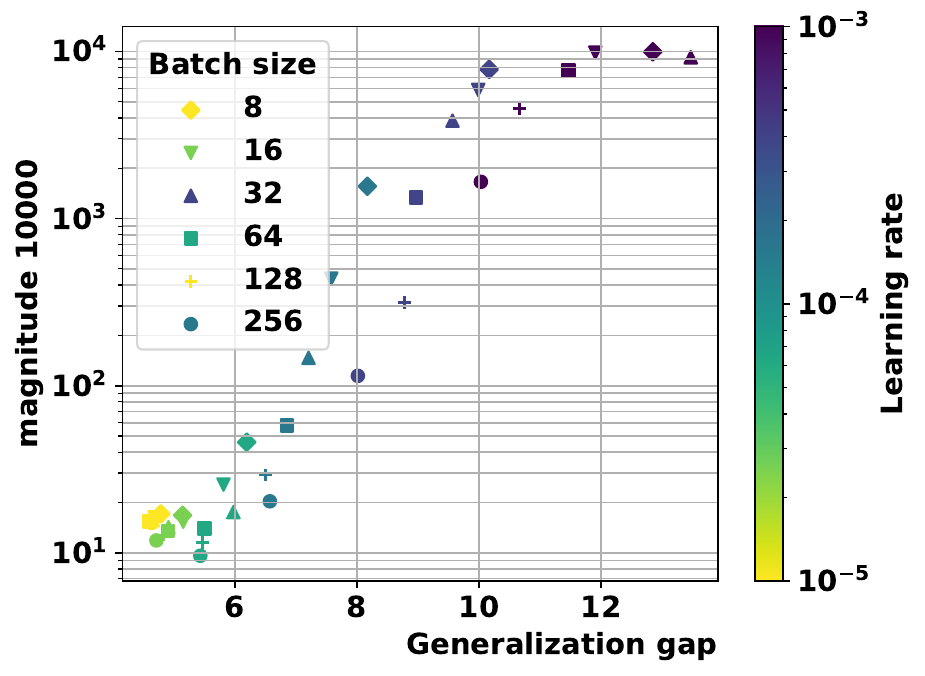} &
        \includegraphics[width=0.25\textwidth]{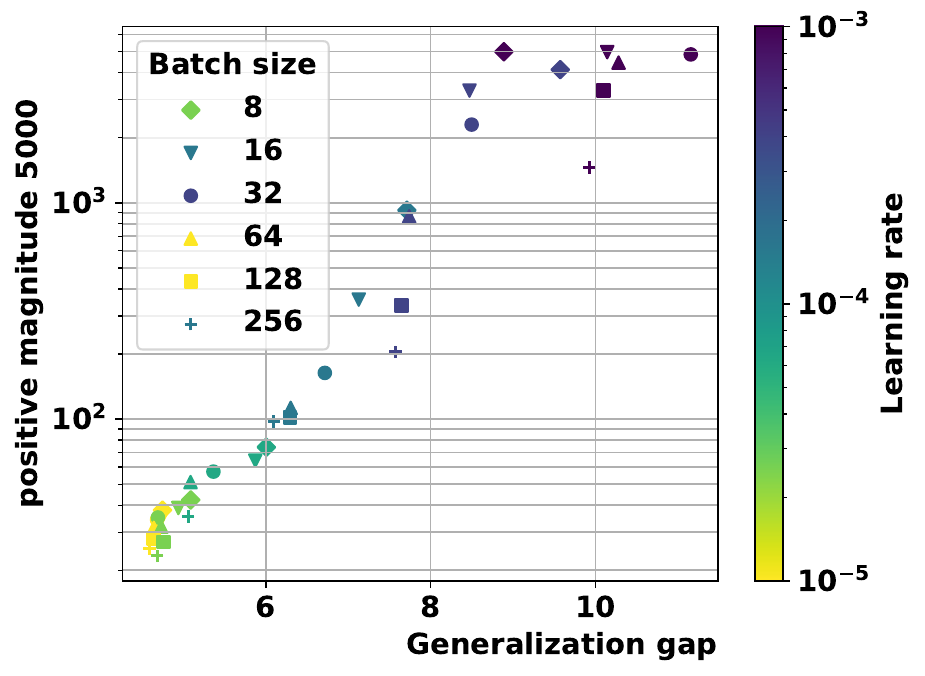} &
        \includegraphics[width=0.25\textwidth]{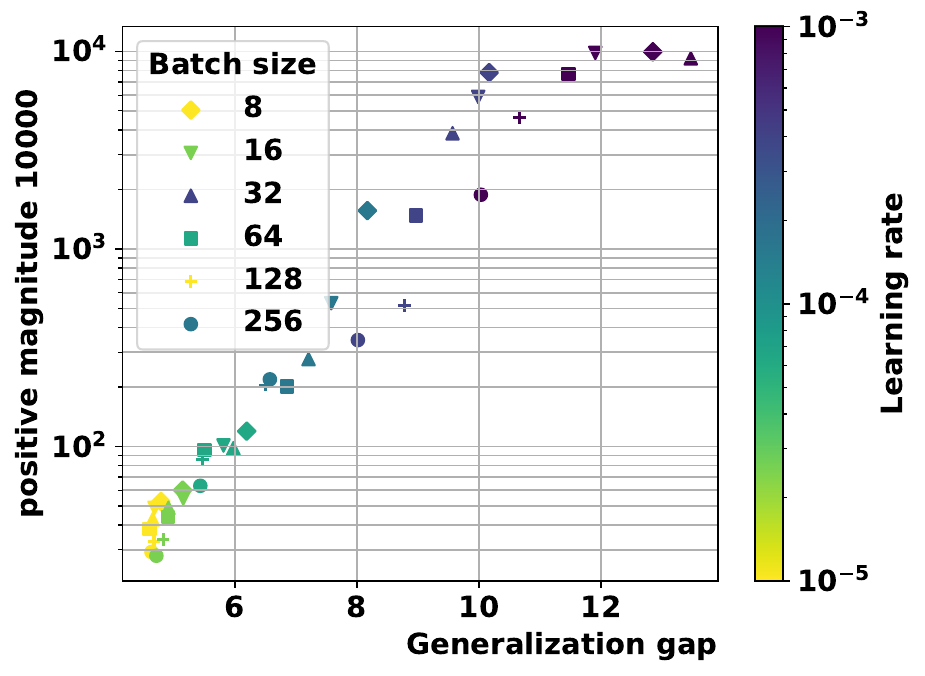}
        \\
        (a) & (b) & (c) & (d)\\
    \end{tabular}
    
    \caption{\textbf{Extended complexity measures vs. the generalization gap} 
    We compare the original topological ocmplexity measure $\mathrm{Mag_{5000}}$ (a) and $\mathrm{PMag_{5000}}$ (c) with the extended complexity measures $\mathrm{Mag_{10000}}$ (b) and $\mathrm{PMag_{10000}}$ (d)  for a ViT trained on CIFAR10. }
    \label{fig:magnitude_traj_ours}
\end{figure*}

\subsubsection{Neural Network Regularization}

Utilising the magnitude approximation described in Section \ref{sec:netrowk_regularization}, we train five neural networks each with two fully connected hidden layers on the MNIST dataset for 2000 epochs, using cross entropy loss on MNIST. We train the models with a scalar multiple of the magnitude of the weights as a penalty term. One of the networks we train (with a regularization constant of 0) corresponds to an unregularised model.
We then evaluate the differences in magnitude as well as train and test loss for each model.

Our results are shown in Table \ref{table:mag-regularisation}. We first observe that as expected, adding a magnitude-based penalty term causes the network's magnitude to decrease. More interestingly magnitude regularization causes the neural network to perform better. This increase in performance occurs both in terms of test loss and generalization error, with the unregularised model recording both the largest test loss and generalization gap. It is also interesting to note that the generalization appears to increase consistently with the strength of regularization, whereas test loss appears to have an optimal strength of regularization at $\lambda = 0.5$.

\begin{table}
\centering
    \begin{tabular}{|l|l|l|l|l|}
    \hline
         $\lambda$ & Train. Loss & Test Loss & Gap & Magnitude \\
         \hline
         0 & 0.0021 & 0.0757 & 0.0736 & 1.5810 \\
         0.1 & 0.0041 & 0.0641 & 0.0600 & 1.1567 \\
         0.2 & 0.0061 & 0.0607 & 0.0546 & 1.1293 \\
         0.5 & 0.0103 & 0.0602 & 0.0499 & 1.0842 \\
         1 & 0.0167 & 0.0631 & 0.0464 & 1.0668 \\
         \hline
    \end{tabular}
    \caption{Magnitude and performance of Neural Networks after training to minimise $\text{Training Loss}(\text{weights}) + \lambda \Mag(\text{weights})$. }
    \label{table:mag-regularisation}
\end{table}

\subsubsection{Clustering}

The results of using the Clustering algorithms described in the previous section are presented in Figure \ref{fig:clustering_results}. We note that our algorithm performs satisfactory, providing clustering better than Agglomerative, $k$-means and DBSCAN. 

\begin{figure}[h]
    \centering
    \includegraphics[width=0.85\linewidth]{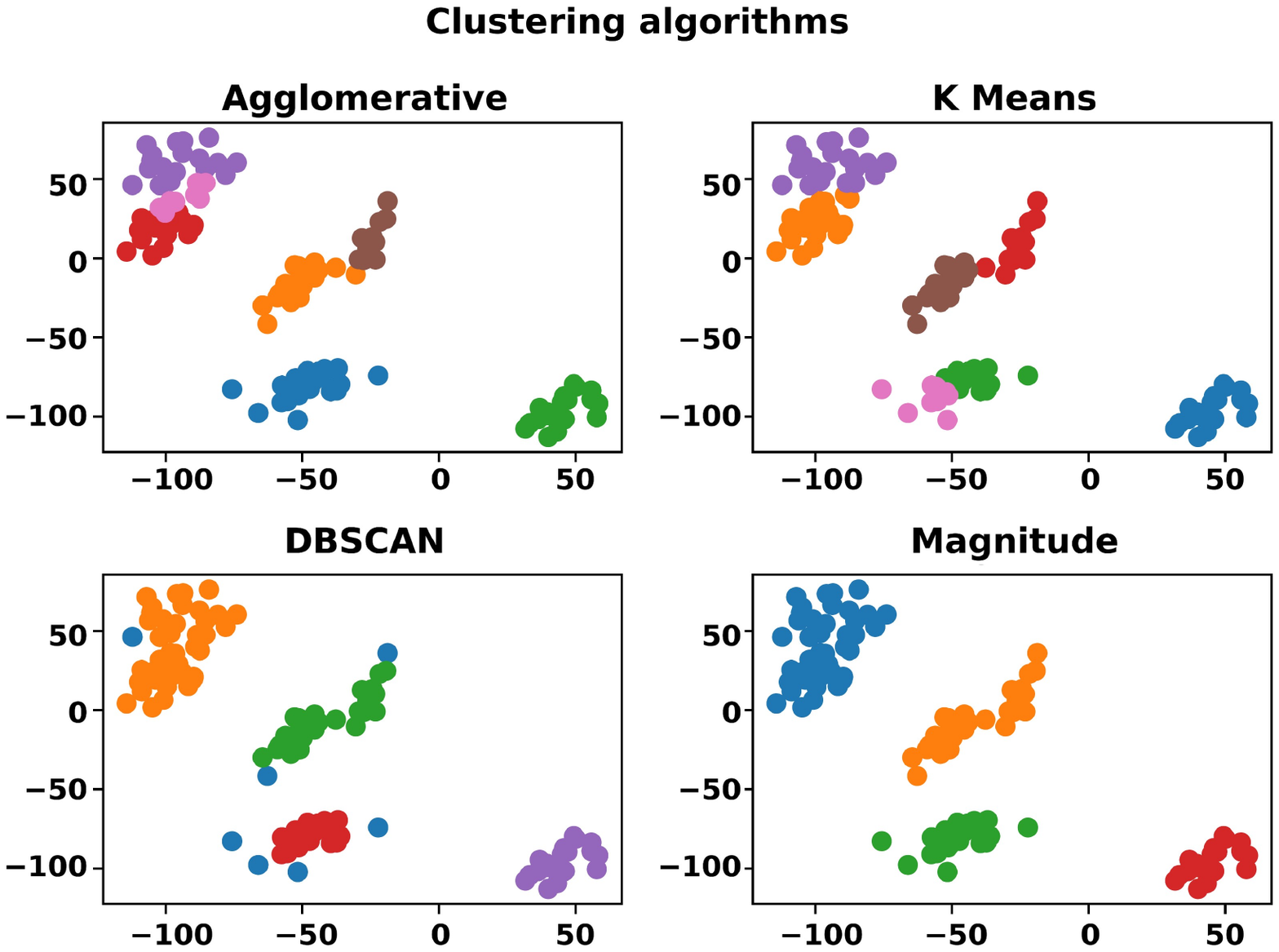}
    \caption{Results of applying the Magnitude clustering algorithm to an artificial dataset. We can see that the magnitude-based algorithm manages to find natural clusters and is able to determine a suitable number of clusters, whereas K-means and hierarchical clustering (Ward clustering in \texttt{scikit-learn}) need the user to determine this. A main difference between DBSCAN and the magnitude algorithm is in the treatment of outliers.}
    \label{fig:clustering_results}
\end{figure}

\section{Related work}
\label{sec:related}

The literature relevant to magnitude and machine learning has already been discussed in previous sections. We have studied closely the relation of magnitude to generalization~\cite{andreeva2024topological,andreeva_metric_2023}. In other works diversity of latent representations have been recently explored in~\cite{limbeck2024metric}. Magnitude based clustering has been suggested in \cite{o2023magnitude}. The algorithms proposed make use of a similar quantity called alpha magnitude \cite{o2023alpha}, but unlike ours, requires multiple parameters as input. 

In the mathematical literature, we observe that recent developments such as Magnitude for graphs~\cite{leinster2019magnitude} and relation between Magnitude and entropy~\cite{chen2023categorical} are likely to be of interest in machine learning. As is the interpretation of magnitude as dual of Reproducing Kernel Hilbert Space~\cite{meckes2015magnitude}.

The computational problem can be seen as solving the linear system $\zeta w=\mathbb{1}$, with $\mathbb{1}$ as a vector of all $1$. Notice that our matrix $\zeta$ is symmetric positive definite but also a dense matrix. The iterative normalization algorithm we proposed bears resemblance to basis pursuit and other algorithms in compressive sensing~\cite{foucart2013mathematical}, but the relation is not yet clear.

\section{Conclusion}
In this paper, we introduced fast and scalable methods for approximating metric magnitude. We provided novel applications to deep learning and clustering, and extended existing ones, leading to better overall performance. We expect that this work will enable the wider use of this versatile geometric concept in machine learning and optimization.
\bibliography{aaai25}

\newpage
% !TEX root ../main.tex

\pagebreak 
\newpage 

\appendix

\section{Appendix}\label{app:1}

\subsection{Proofs of theorems}

\myparagraph{Proof of Theorem~\ref{thm:non-submodularity}}

\begin{proof}

We note that $X$ is a homogeneous metric space, thus using Proposition 2.1.5 from \cite{leinster2013magnitude} we can calculate 

    \[ Mag(X) = \frac{2D}{1+e^{-2t}+2(D-1)e^{-t\sqrt{2}}} \]

    Next, we consider the similarity matrix of $X \cup \{0\}$ to calculate its magnitude. 

    \[ \zeta_{X \cup \{0\}} = \begin{pmatrix}
        A_1 & A_2 \\
        A_3 & A_4
    \end{pmatrix} \]

    where $A_1$ is the similarity matrix for $X$,

    \[ A_2 = \begin{pmatrix}
        e^{-t} \\
        e^{-t} \\
        \vdots \\
        e^{-t} 
    \end{pmatrix} \quad A_3 = A_2^T \quad A_4 = \begin{pmatrix}
        1
    \end{pmatrix} .\]

    We will use the inverse formula:

    \[ \begin{pmatrix}
        A_1 & A_2 \\
        A_3 & A_4
    \end{pmatrix}^{-1} \] 
    
    \[ = \begin{pmatrix}
        A_1^{-1} + A_1^{-1}A_2B^{-1}A_3A_1^{-1} & -A_1^{-1}A_2B^{-1} \\
        -B^{-1}A_3A_1^{-1} & B^{-1}
    \end{pmatrix} \]

    where $B = A_4-A_3A_1^{-1}A_2$

    As $X$ is homogeneous, we observe that the sum of rows of $A_1^{-1}$ give $\frac{\Mag(X)}{D}$, so each entry of $A_1^{-1}B$ is $e^{-1}\frac{\Mag(X)}{D}$, and by symmetry the same is true for $CA^{-1}$. Thus $B = \frac{1}{1-e^{2}\Mag(X)}$. It then follows that the magnitude of $|X \cup \{0\}|$ is 

    \[ \frac{(1-2e^{-t})\Mag(X)+1}{1-e^{-2t}\Mag(X)}. \] 

    When this expression is expanded, L'Hopital's rule then gives that

    \[ \lim_{D \to \infty} \Mag(X \cup \{0\}) - \Mag(X) = \frac{8(e^{2t} - e^{t(1+\sqrt{2})})^2}{8(e^{4t}-e^{t(2+\sqrt{2})})} = \frac{(e^t-e^{t\sqrt{2}})^2}{e^{2t}-e^{t\sqrt{2}}}. \]

\end{proof}

\begin{definition}
    Let $x,y \in \mathbb{R}^n$ be such that $x_1 \geq x_2 \geq ... \geq x_n$, $y_1 \geq y_2 \geq ... \geq y_n$ and $\sum_{i=1}^n x_i = \sum_{i=1}^n y_i$. Then we say that $x$ majorises $y$ if for all $k = 1,...,n$

    \[ \sum_{i=1}^k x_i \geq \sum_{i=1}^k y_i. \]
\end{definition}

\begin{theorem}[Karamata's inequality]
    Let $I$ be an interval on $\mathbb{R}$ and let $f:I \to \mathbb{R}$ be concave. If $x_1,...,x_n$ and $y_1,...,y_n$ are numbers in $I$ such that $(x_1,...,x_n)$ majorises $y_1,...,y_n$ then

    \[ f(x_1) + ... + f(x_n) \leq f(y_1) + ... + f(y_n). \]
\end{theorem}

\myparagraph{Proof of Theorem~\ref{thm:submod-R}}
\begin{proof}
Given a set $X$, we can write $B = \{b_1 < ... < b_n\} = X \cup \{-\infty, \infty\}.$ The formula provided in corollary 2.3.4 from \cite{leinster2013magnitude} gives that

\[ \Mag(X) = \sum_{i=1}^n \tanh\left(\frac{b_{i+1} - b_i}{2}\right) - 1. \]

Note that here we define $\tanh{\infty} = 1$.

    Then given the points $x_1 < x_2$ such that $x_1 \in (b_j, b_{j+1})$ and $x_2 \in (b_k, b_{k+1})$, we calculate that

\[ \Mag(X \cup \{x_1\}) \]
\[ = \sum_{i=1}^n \tanh\left(\frac{b_{i+1} - b_i}{2}\right) - 1 - \tanh\left(\frac{b_{j+1} - b_j}{2}\right)\]
    \[ + \tanh\left(\frac{b_{j+1} - x_1}{2}\right) + \tanh\left(\frac{x_1 - b_j}{2}\right).\]

    \[ \Mag(X \cup \{x_2\}) \]
    \[= \sum_{i=1}^n \tanh\left(\frac{b_{i+1} - b_i}{2}\right) - 1 - \tanh\left(\frac{b_{k+1} - b_k}{2}\right)\]
    \[ + \tanh\left(\frac{b_{k+1} - x_2}{2}\right) + \tanh\left(\frac{x_2 - b_k}{2}\right).\]

    If $j \neq k$ then

    \[ \Mag(X \cup \{x_1, x_2\}) = \sum_{i=1}^n \tanh\left(\frac{b_{i+1} - b_i}{2}\right) - 1 \]
    \[ - \tanh\left(\frac{b_{j+1} - b_j}{2}\right) + \tanh\left(\frac{b_{j+1} - x_1}{2}\right) + \tanh\left(\frac{x_1 - b_j}{2}\right) \]
    \[ - \tanh\left(\frac{b_{k+1} - b_k}{2}\right) + \tanh\left(\frac{b_{k+1} - x_2}{2}\right) + \tanh\left(\frac{x_2 - b_k}{2}\right). \]

    It follows that 
    \[ \Mag(X \cup \{x_1\}) + \Mag(X \cup \{x_2\}) \]
    \[ = \Mag(X \cup \{x_1, x_2\}) + \Mag(X).\]

    If $j = k$, then 

    \[ |X \cup \{x_1, x_2\}| = \sum_{i=1}^n \tanh\left(\frac{b_{i+1} - b_i}{2}\right) - 1 \]
    \[  - \tanh\left(\frac{b_{j+1} - b_j}{2}\right) + \tanh\left(\frac{b_{j+1} - x_2}{2}\right) \]
    
    \[ + \tanh\left(\frac{x_2 - x_1}{2}\right) + \tanh\left(\frac{x_1 - b_j}{2}\right). \]

    So

    \[ \Mag(X \cup \{x_1\}) + \Mag(X \cup \{x_2\}) - \Mag(X \cup \{x_1, x_2\}) - \Mag(X) \]
    \[ = \tanh\left(\frac{b_{j+1} - x_1}{2}\right) + \tanh\left(\frac{x_2 - b_j}{2}\right) - \]
    
    \[ \tanh\left(\frac{b_{j+1} - b_j}{2}\right) - \tanh\left(\frac{x_2 - x_1}{2}\right).  \]

    We observe that $\tanh(x)$ is concave on $x \geq 0$ and that since $b_{j+1} - b_j \geq b_{j+1} - x_1$ and $b_{j+1} - b_j \geq x_2 - b_j$ and $b_{j+1} - b_j + x_2 - x_1 = b_{j+1} - x_1 + x_2 - b_j$, $(x_2-x_1, b_{j+1} - b_j)$ majorises $(b_{j+1} - x_1, x_2 - b_j)$. Thus by Karamata's inequality,

    \[ \tanh\left(\frac{b_{j+1} - x_1}{2}\right) + \tanh\left(\frac{x_2 - b_j}{2}\right) - \]
    
    \[ \tanh\left(\frac{b_{j+1} - b_j}{2}\right) - \tanh\left(\frac{x_2 - x_1}{2}\right) \geq 0. \]

    Thus it follows that 

    \[ \Mag(X \cup \{x_1\}) + \Mag(X \cup \{x_2\}) \geq \Mag(X \cup \{x_1, x_2\}) + \Mag(X). \]

    Hence magnitude on $X$ is submodular.
\end{proof}

\begin{theorem}
    \label{thm:3pts}
    Every set of 3 points with the magnitude function is submodular.
    \end{theorem}
\begin{proof}
    Let $X = \{x_1,x_2,x_3\}$ be a 3-point space and $d$ be a metric on $X$. Then magnitude is submodular on the metric space $(X,d)$.

    Write $d_1 = d(x_1,x_2), d_2=d(x_1,x_3), d_3=(x_2,x_3)$. 
Take $a = \text{max} \{d_1,d_2,d_3\}$ and $b = \text{min} \{d_1,d_2,d_3\}$. We note that $0 < b < a$.

Let $Y$ be the space containing the 2 points $b$ apart. Then by the above corollary for any 2-point subspace $X$, $Mag(X) \geq |Y|$.

We then calculate the magnitude of $Y$ as

\[ \left| \begin{pmatrix}
1 & e^{-b} \\
e^{-b} & 1
\end{pmatrix}\right| = \frac{2e^b}{e^b+1}. \]

Furthermore, adding in the remaining point, the distances from that point and the original pair of points are at most $a$. So then again using Lemma 2.2.5 in \cite{leinster2013magnitude}, the magnitude of $X$ is at most:

\[ \left| \begin{pmatrix}
1 & e^{-b} & e^{-a} \\
e^{-b} & 1 & e^{-a} \\
e^{-a} & e^{-a} & 1
\end{pmatrix}
\right|  = \frac{(3e^{2a} - 4e^a)e^b+e^{2a}}{(e^{2a}-2)e^b + e^{2a}}. \]

It then follows that the difference in magnitude between any 3 and 2-point set is at most 

\[ \frac{(3e^{2a} - 4e^a)e^b+e^{2a}}{(e^{2a}-2)e^b + e^{2a}} - \frac{2e^b}{e^b+1} \]

\[ = \frac{(e^{2a}-4e^a+4)e^{2b}+2(e^{2a}-2e^a)e^b+e^{2a}}{(e^{2a}-2)e^{2b}+2(e^{2a}-1)e^b+e^{2a}}. \]

A similar argument shows that the increase in magnitude from a 1-point to a 2-point set is at most

\[ \frac{e^a-1}{e^a+1}. \]

Then suppose

\[ \frac{(e^{2a}-4e^a+4)e^{2b}+2(e^{2a}-2e^a)e^b+e^{2a}}{(e^{2a}-2)e^{2b}+2(e^{2a}-1)e^b+e^{2a}} \geq \frac{e^a-1}{e^a+1}. \]

This implies that 

\[\ ((e^{2a}-4e^a+4)e^{2b}+2(e^{2a}-2e^a)e^b+e^{2a})(e^a+1) \]

\[ \geq ((e^{2a}-2)e^{2b}+2(e^{2a}-1)e^b+e^{2a})(e^a-1)  \]

\[ \implies -2(e^b-1)(-e^{a+b}+e^{2a+b}+e^{2a}-e^b) \geq 0. \]

We note that as $a > b > 0$, $(-e^{a+b}+e^{2a+b}+e^{2a}-e^b) \geq 0$, thus $(e^b-1) \leq 0$ for $b \geq 0$ which is a contradiction.

Thus the increase from 1 to 2 points must be greater than or equal to the increase from 2 to 3 points.
\end{proof}

\section{More experimental results and details}\label{app:2}

\subsubsection{Larger datasets}
In Figure \ref{fig:20000_points_iterative_inversion}, we see a comparison between matrix inversion and Iterative Normalizaion for $2 \times 10^{4}$ points sampled from $\mathcal{N}(0,1)$ in $\mathbb{R}^2$ over 5 runs. Iterative Normalization is run for $10$ iteration, as we observe fast convergence towards the true magnitude value. 

The experiment was ran on a NVIDIA 2080Ti GPU with 11GB RAM.
\begin{figure}[hbt!]
    \centering
    \includegraphics[width=0.99\linewidth]{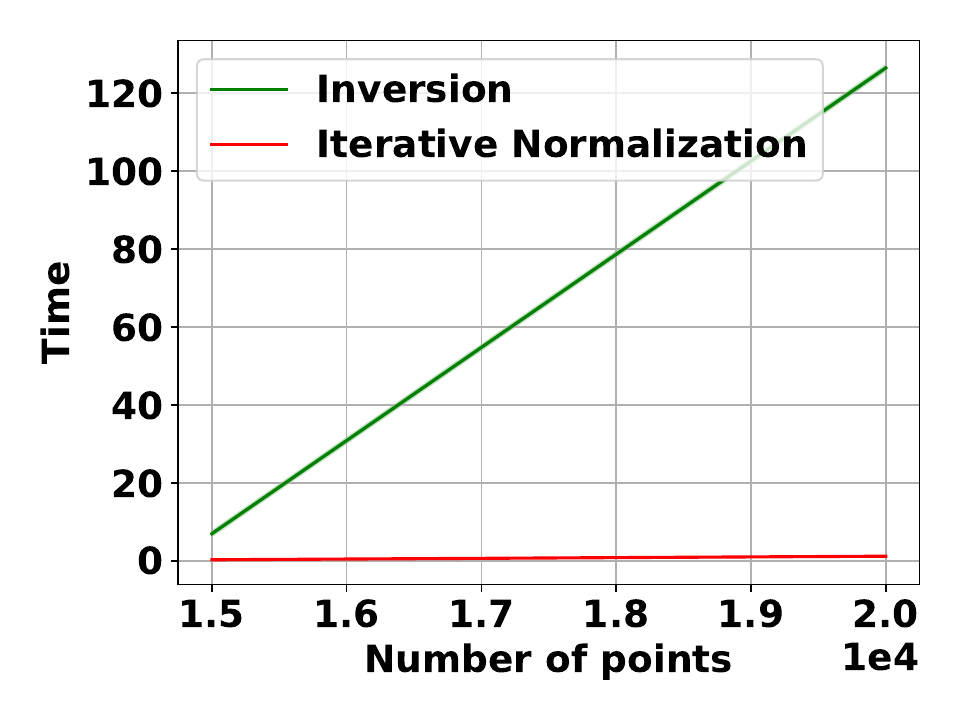}
    \caption{Time is measured in seconds. It takes $1.12$ seconds for Iterative Normalization to execute for $2 \times 10^4$ points, while Inversion requires $126.9$ seconds. }
    \label{fig:20000_points_iterative_inversion}
\end{figure}

\subsection{Further investigation of SGD algorithm}
We experimented with multiple different batch sizes, but full size of the dataset, or Gradient Descent (GD) appears to achieve fastest convergence. For completeness, here we report the convergence results when we vary the batch size. Note that while this method appears to be slower than Iterative Normalization, it can be used when the size of the dataset cannot fit into memory.

Batch sizes are $=\{8, 16, 32, 64, 128, 256\}$ and show in how many iterations the algorithm converges, and learning rate is fixed at $0.01$. Figure \ref{fig:sgd_vary_batch_size} shows mean and standard deviation, repeated over 10 runs, for 50 iterations, for a dataset with $2000$ points sampled from $\mathcal{N}(0,1)$ in $\mathbb{R}^2$.

It appears that all batch sizes tend to converge towards the true magnitude value after iteration 20.

Experiments ran on a NVIDIA 2080Ti GPU with 11GB RAM.

\begin{figure}[hbt!]
    \centering
    \includegraphics[width=0.99\linewidth]{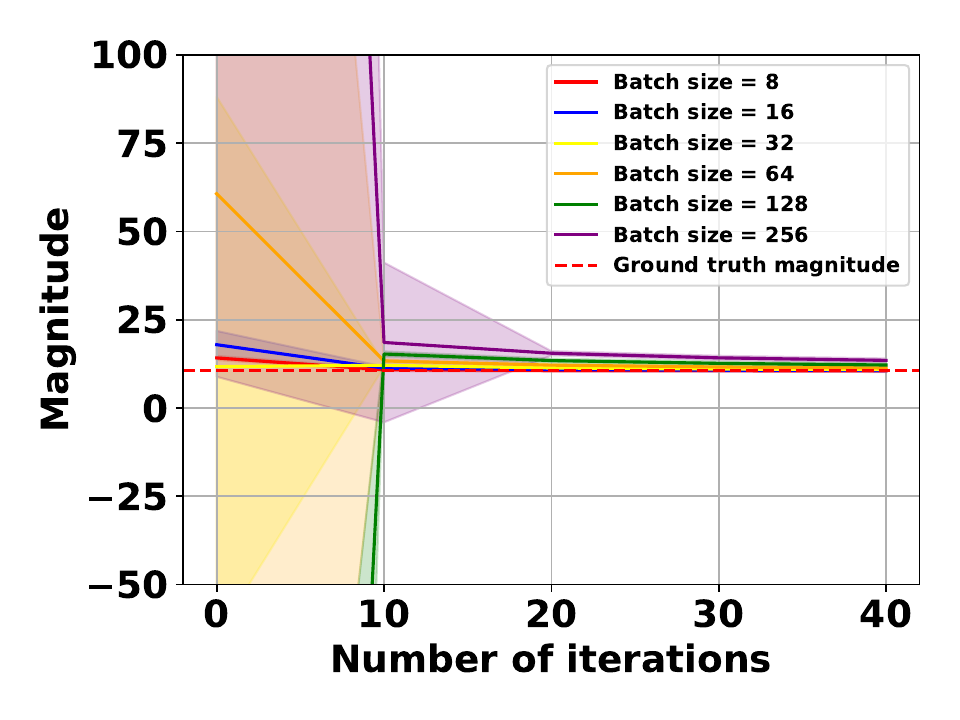}
    \caption{Varying the batch sizes and its effect on convergence towards the true magnitude value.}
    \label{fig:sgd_vary_batch_size}
\end{figure}

\subsection{Experimental details}

\subsubsection{Algorithm comparison}
We use \texttt{PyTorch}'s GPU implementation for matrix inverison. The SGD experiments used a learning rate of $0.01$ and momentum of $0.9$. We set the size of the batch to equal the size of the dataset for the GD experiments in the main paper.

\subsubsection{Training trajectories and generalization}
We use a modified version of the ViT for small datasets as per \cite{raghu2021vision}. The implementation is based on the \cite{gani2022train}, which is based on the \texttt{timm} library with the architecture parameters as follows: depth of 9, patch size of 4, token dimension of 192, 12 heads, MLP-Ratio of 2, resulting in 2697610 parameters in total, as described in more detail in \cite{andreeva2024topological}.

We start from a pre-trained weight vector, which achieves high training accuracy on the classification task. By varying the learning rate in the range $[10^{-5}, 10^{-3}]$ and the batch size between $[8,256]$, we define a grid of $6\times 6$ hyperparameters. For each pair of batch size and learning rate, we compute the training trajectory for $10^4$ iterations. We use the Adam optimizer \cite{kingma2017adam}. We compute the data-dependent pseudometric, first defined in \cite{dupuis2023generalization} by $\rho_S^{(1)}(w,w') = r^{-1}||\boldsymbol{L}_S(w) - \boldsymbol{L}_S(w')||_1 $, to obtain a distance matrix. Then we proceed to compute the quantities of interest $\mathrm{Mag_n}$ and $\mathrm{PMag_n}$ for $n=\{5000, 7000, 10000\}$, using the distance matrix as derived from the pseudometric $\rho_S^{(1)}$. We set the magnitude scale $t = \sqrt{r}$, where $r$ is the size of the training set ($r = 50000$ for CIFAR10). This value is motivated by the theory in \cite{andreeva2024topological}, and for a fair comparison with their methods. We then compute the granulated Kendall's coefficients ($\boldsymbol{\psi}_{\text{lr}}$ and $\boldsymbol{\psi}_{\text{bs}}$ for the learning rate and for batch size respectively, $\boldsymbol{\Psi}$, which is the Average Kendall coefficient (the average of $\boldsymbol{\psi}_{\text{lr}}$ and $\boldsymbol{\psi}_{\text{bs}}$) \cite{gastpar2023fantastic}), which are more relevant than the classical Kendall's coefficient for capturing causal relationships; and Kendall tau ($\tau$).

Experiments ran on a NVIDIA 2080Ti GPU with 11GB RAM. 

\subsubsection{Regularization}
We train five neural networks each with two fully connected hidden layers on the MNIST dataset for 2000 epochs, using cross entropy loss on MNIST and a learning rate of $0.001$. 

Experiments ran on NVIDIA GeForce GTX 1060 6GB GPU.

\begin{figure}[hbt!]
    \centering
    \includegraphics[width=0.8\linewidth]{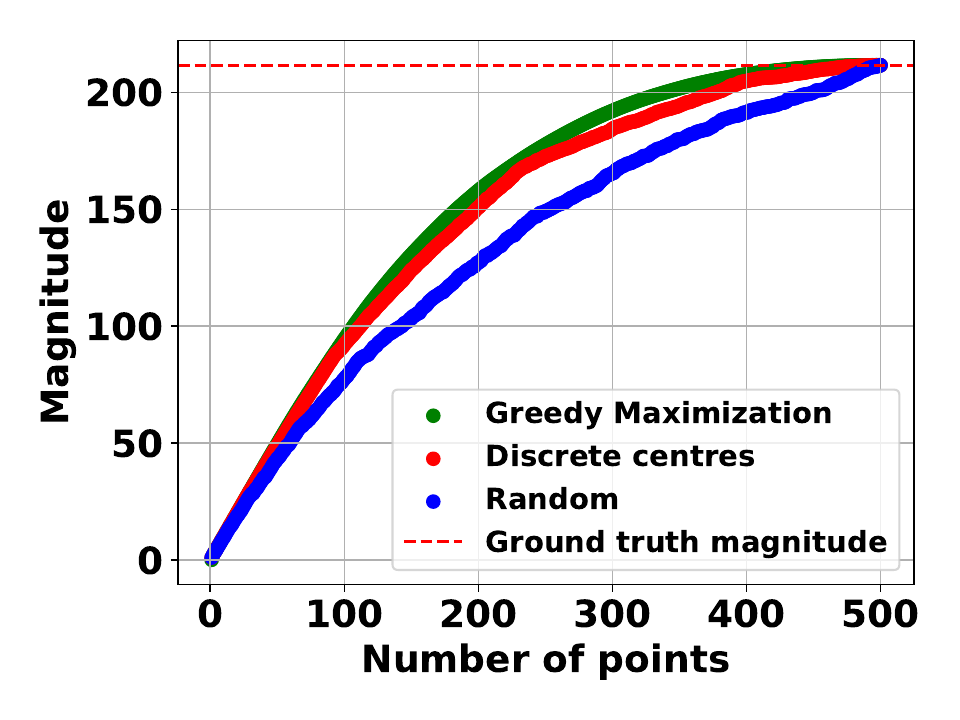}
    \caption{Comparison of subset selection on the Swiss roll with subsample of 500 points.}
    \label{fig:swiss_roll_example}
\end{figure}

\subsubsection{Clustering}
For DBSCAN, we used $\epsilon = 10$ and minimum number of clusters = $2$. For
$k$-means and Agglomerative clustering, the minimum number of clusters was set to the number of cluster centers used to generate the dataset.

Experiments ran on Intel Xeon CPU E5-2603 v4 CPU with 3GB memory.

\subsection{More experimental Results - Clustering}
In Figure \ref{fig:mega-clustering-more-results}, we show more results of the magnitude clustering algorithm using a number of different random seeds for dataset generation.

\begin{figure*}
\begin{tabular}{cc}
    \centering
    \includegraphics[width=0.5\linewidth]{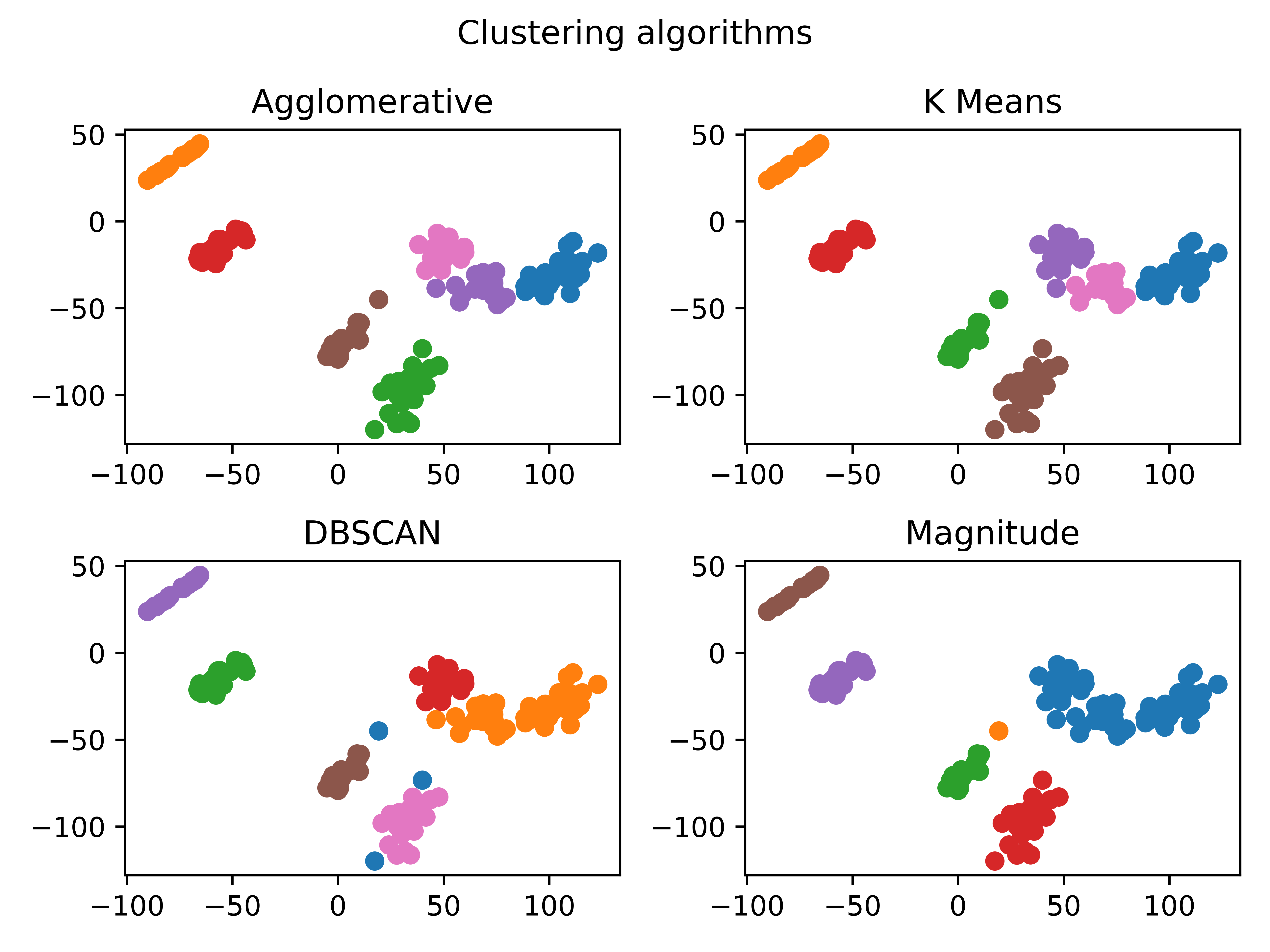}

    &
    \includegraphics[width=0.5\linewidth]{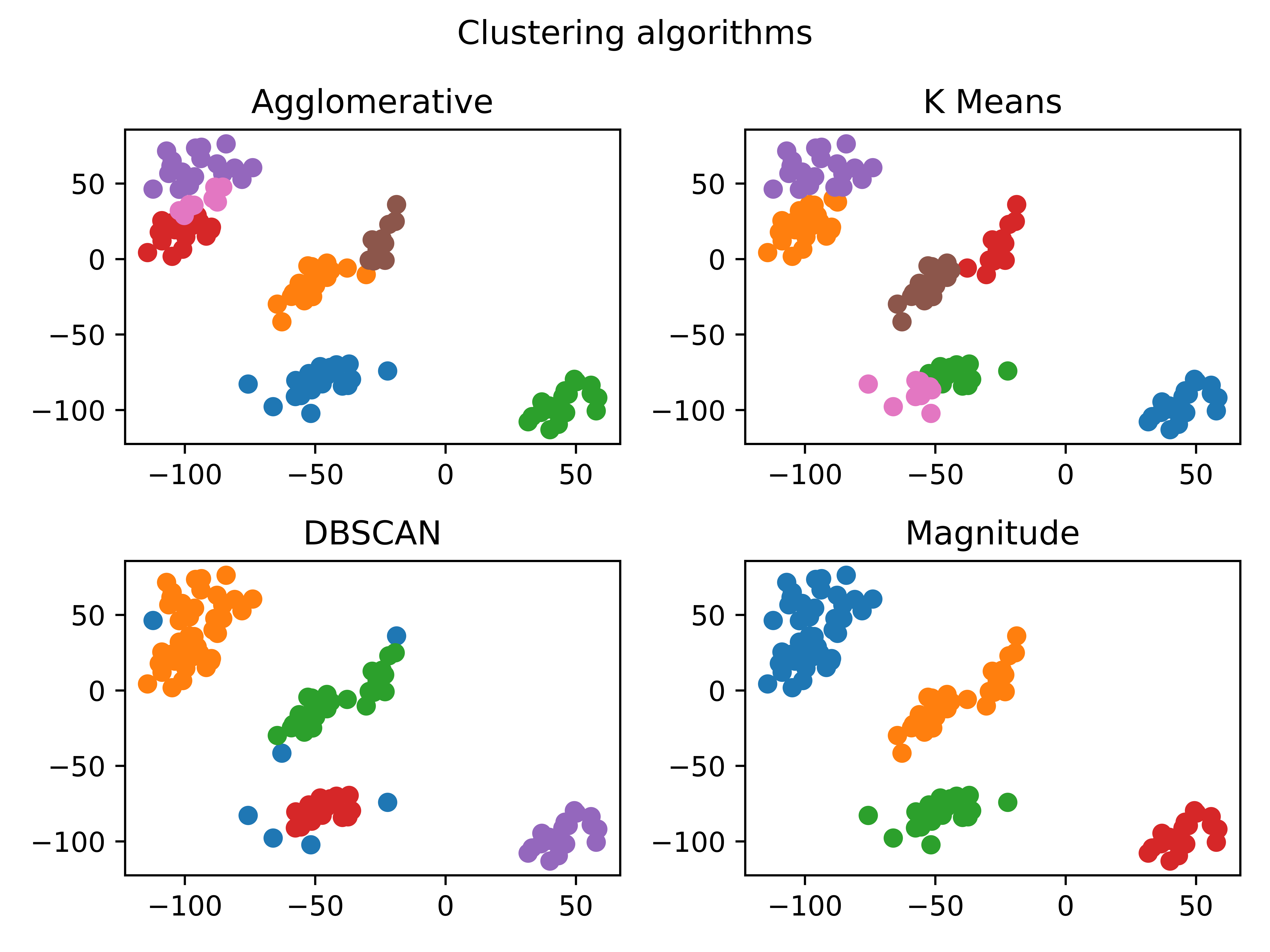}

    \\
    \includegraphics[width=0.5\linewidth]{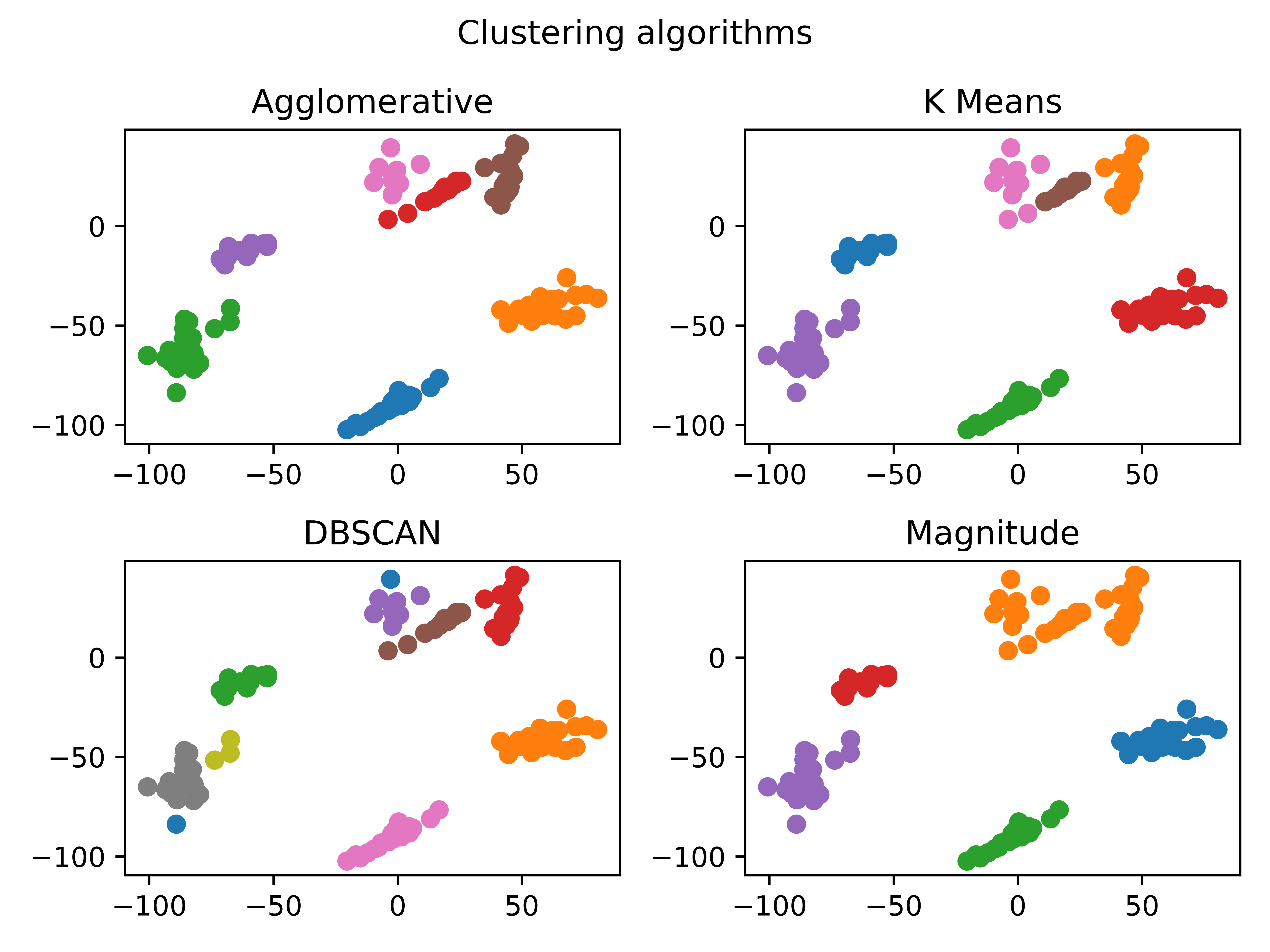}

    &
    
    \centering
    \includegraphics[width=0.5\linewidth]{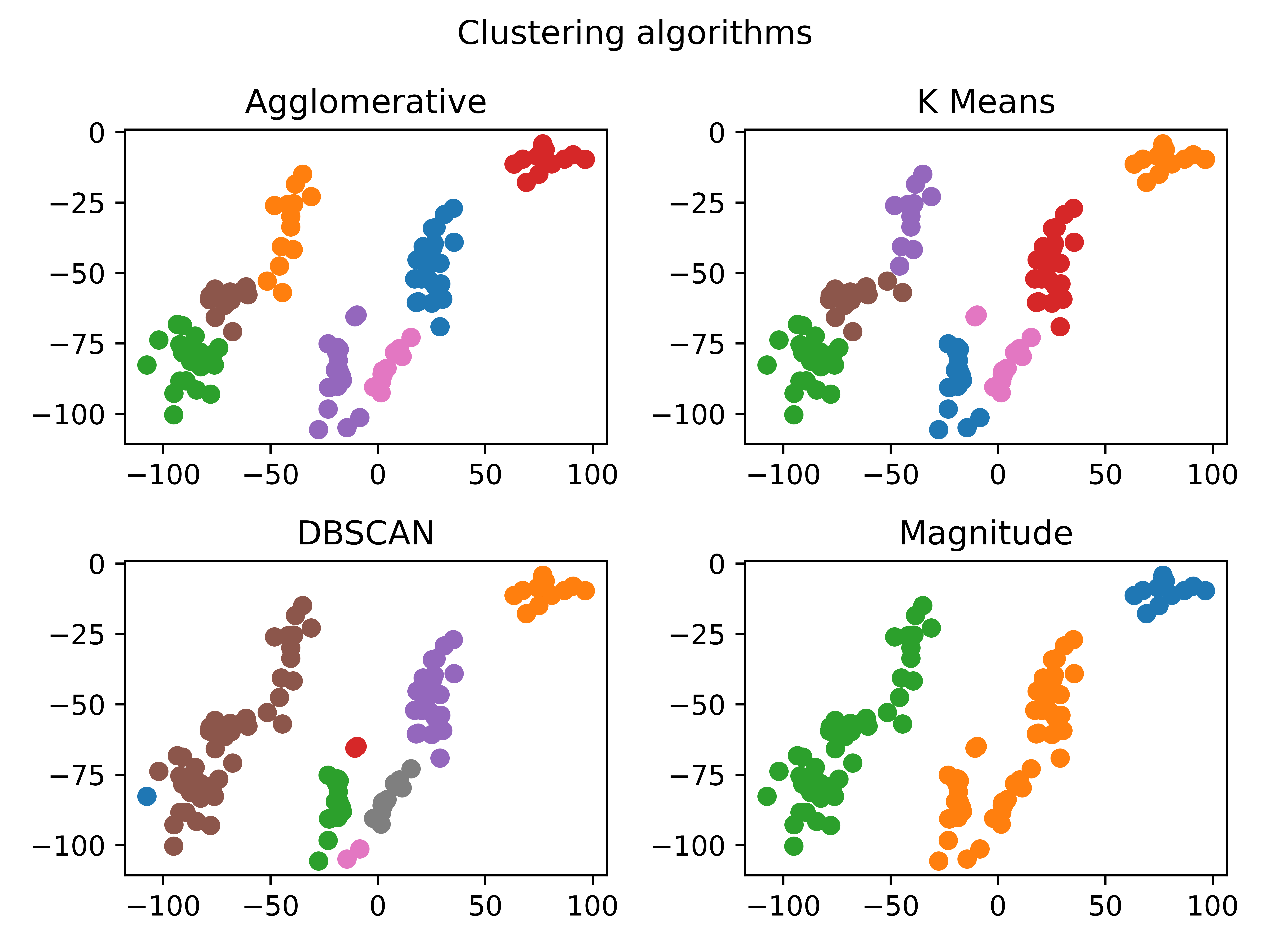}
\end{tabular}
\caption{Magnitude clustering algorithm versus conventional algorithms on randomly generated datasets. The difference between the plots comes from using a different random seed to generate the datasets. We note that the magnitude algorithm consistently identifies a reasonable number of clusters and provides a sensible cluster assignment to each point.}
    \label{fig:mega-clustering-more-results}
\end{figure*}

\subsection{Subset selection}
Figure \ref{fig:swiss_roll_example} shows an example of a synthetic dataset called the Swiss roll. The Discrete centers algorithm produces a hierarchy which provides almost the same approximation as the Greedy Maximization algorithm for a fraction of the computational cost.
In Figure \ref{fig:magnitude_sklearn_datasets1} we see a comparison between Greedy Maximization, Discrete Centers and selecting points at random (Random) for the 3 standard \texttt{scikit-learn} datasets Iris, Breast cancer and Wine dataset. We have used the entire dataset for generating the plot. In Figure \ref{fig:magnitude_sklearn_datasets2} we see the same comparison, but with a random subset of 500 points from MNIST, CIFAR10 and CIFAR100. We have reduced the dimensions of each dataset using PCA to 100.

 \begin{figure*}[hbt!]
    \begin{tabular}{ccc}
        \includegraphics[width=0.33\textwidth]{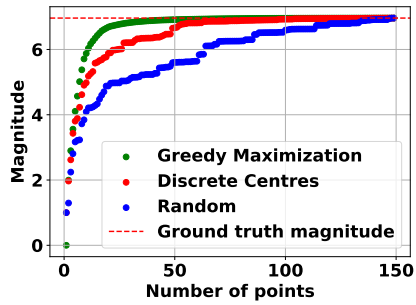} & \includegraphics[width=0.33\textwidth]{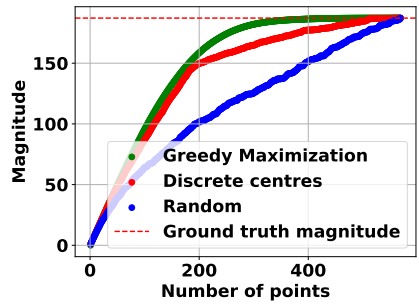} &
        \includegraphics[width=0.33\textwidth]{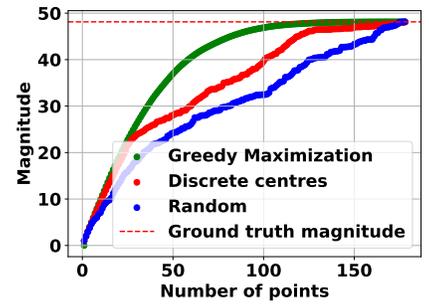}  
        \\
        \\
        (a) & (b) & (c)\\
    \end{tabular}
    
    \caption{\textbf{Discrete centers are close to Greedy Maximization at a fraction of the computational cost and better than random.} In plot (a) we have the Iris dataset, in plot (b) the Breast cancer dataset, in plot (c) the Wine dataset.}
    \label{fig:magnitude_sklearn_datasets1}
\end{figure*}

 \begin{figure*}
    \begin{tabular}{ccc}
      \includegraphics[width=0.33\textwidth]{Figures/Real_world_datasets/mnist_real_subsample_500_final.pdf} &         \includegraphics[width=0.33\textwidth]{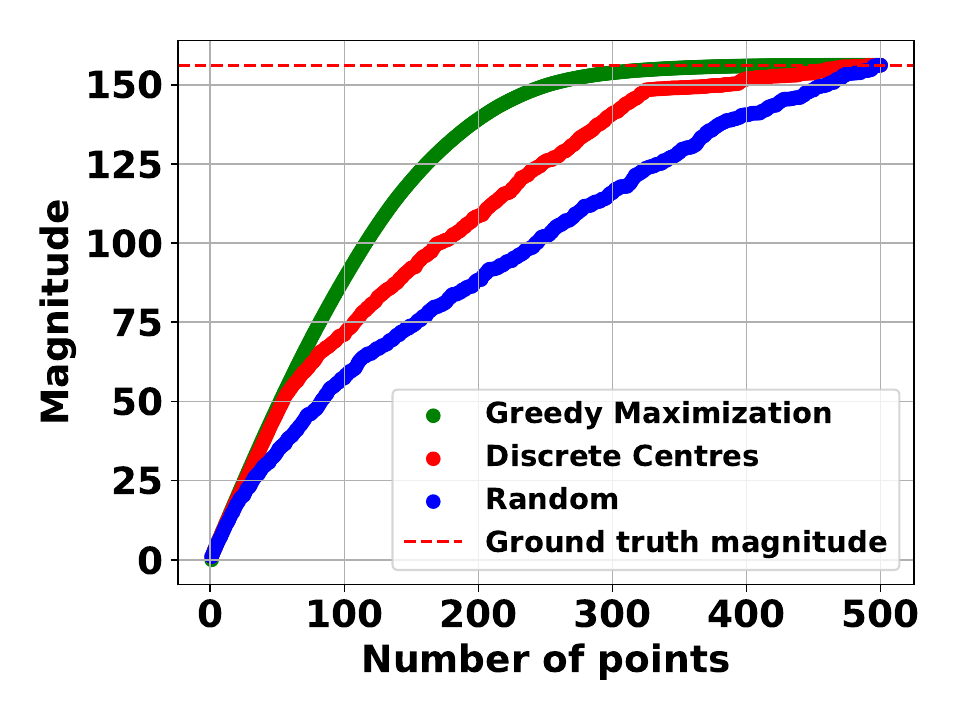} &
        \includegraphics[width=0.33\textwidth]{Figures/Real_world_datasets/cifar100_real_subsample_500_final.pdf}
        \\
        \\
        (a) & (b) & (c) \\
    \end{tabular}
    
    \caption{\textbf{Discrete centers are close to Greedy Maximization at a fraction of the computational cost and better than random.} In plot (a) we a subsample from MNIST dataset, in plot (b) from CIFAR10, and in plot (c) we see a subsample of CIFAR100.}
    \label{fig:magnitude_sklearn_datasets2}
\end{figure*}

\subsection{Magnitude of a compact space}

For completeness, we provide the formal definition of magnitude for compact sets, and a few important results.

\begin{definition}
    A metric space is positive definite if every finite subspace is positive definite. The magnitude of a compact positive definite space $A$ is
    \begin{equation}
    \mathrm{Mag}(A) = \sup\{ \mathrm{Mag}(B): B \text{ is a finite subspace of } A\} \in [0, \infty].
    \end{equation}
\end{definition}

\begin{definition}
    Let a weight measure for a compact space $A$ be a signed measure $\mu \in M(A)$ such that, for all $a \in A$,
    \begin{equation}
        \int e^{-d(a,b)}d\mu(b) = 1.
    \end{equation}
    Then $\mathrm{Mag}(X) = \mu(A)$ whenever $\mu$ is a weight measure for $A$.
\end{definition}

\begin{theorem}[Theorem 5.4 in \cite{meckes2015magnitude}]
    Let $A \subset \mathbb{R}^n$ be compact and $t \geq 1$. Then
    \begin{align}
        \frac{\mathrm{Mag}(A)}{t} \leq \mathrm{Mag}(tA) \leq t^n\mathrm{Mag}(A)
    \end{align}
\end{theorem}

\begin{theorem}[Theorem 1 in \cite{barcelo2018magnitudes}]
\label{thm:compactboundedmag}
    Let $X$ be a nonempty compact set in $\mathbb{R}^n$. Then
    \begin{align}
        \Mag(tX) \rightarrow 1 \text{ as } t \rightarrow 0
    \end{align}
    and
    \begin{align}
    t^{-n}\Mag(tX) \rightarrow \frac{Vol(X)}{n!\omega_{n}} \text{ as } t \rightarrow \infty.
    \end{align}
\end{theorem}

\end{document}